\documentclass[a4paper, 10pt]{article} % Anonymized submission
\usepackage{amsmath, amsthm, amssymb}
\title{A simpler spectral approach for clustering in directed networks}
\usepackage{float, upgreek, algorithm, algorithmic, verbatim, eucal, mathtools, natbib}
\bibpunct{(}{)}{;}{a}{,}{,}

\usepackage[colorlinks,
linkcolor=blue,
citecolor=blue,
urlcolor=magenta,
linktocpage,
plainpages=false]{hyperref}
\usepackage{graphicx}
\usepackage{mathrsfs}

\newtheorem{theorem}{Theorem}
\newtheorem{lemma}[theorem]{Lemma}
\newtheorem{proposition}[theorem]{Proposition}
\newtheorem{remark}[theorem]{Remark}

\newtheorem{definition}[theorem]{Definition}
\newtheorem{conjecture}[theorem]{Conjecture}

\author{
  Simon Coste
  \thanks{INRIA Paris, France. SC is supported by ERC NEMO, under the European Union’s Horizon 2020 research and innovation programme grant agreement number 788851.} \\ \texttt{simon.coste@inria.fr}
   \and Ludovic Stephan \footnotemark[1] \footnotemark[2] \thanks{Sorbonne Universit\'e, Paris, France} \\ \texttt{ludovic.stephan@ens.fr}
}

\newcommand{\ov}{\mathrm{ov}}
\newcommand{\thresh}{\upvartheta}
\DeclareMathOperator{\card}{Card}
\DeclareMathOperator{\diag}{diag}
\DeclareMathSymbol{I}{\mathalpha}{operators}{`I}
\newcommand{\sigmar}{\sigma_{d}}
\newcommand{\sigmal}{\sigma_{g}}
\newcommand{\Sigmar}{\Sigma_{d}}
\newcommand{\Sigmal}{\Sigma_{g}}
\newcommand{\pl}{p}
\newcommand{\pr}{q}

\begin{document}

\maketitle

\begin{abstract}%
We study the task of clustering in directed networks. We show that using the eigenvalue/eigenvector decomposition of the adjacency matrix is simpler than all common methods which are based on a combination of data regularization and SVD truncation, and works well down to the very sparse regime where the edge density has constant order. 
%This simple approach was largely unnoticed in the mathematics and network science communities. 
Our analysis is based on a Master Theorem describing sharp asymptotics for isolated eigenvalues/eigenvectors of sparse, non-symmetric matrices with independent entries. We also describe the limiting distribution of the entries of these eigenvectors; in the task of digraph clustering with spectral embeddings, we provide numerical evidence for the superiority of Gaussian Mixture clustering over the widely used k-means algorithm. 
\end{abstract}

%\tableofcontents

\section{Introduction}

The proverbial effectiveness of spectral clustering (``good, bad, and spectral'', said \cite{kannan2004clusterings}), observed for long by practitioners, begins to be well-understood from a theoretical point of view. More and more problem-specific spectral algoritms are periodically designed, with better computational performance and accuracy. Most of the theoretical studies used to be concentrated on symmetric interactions and undirected networks, but over the last decade a flurry of works has tackled the directed case: the 2013 survey \cite{malliaros2013clustering} gave an account of the richness of directed network clustering, but since then the field expanded in various directions.

In directed networks, the directionality of interactions is taken into account. This is more realistic from a modeling point of view, but at the cost of intricate technicalities in the analysis. In this context, the aim of this paper is to take a step back at spectral algorithms and convince the readers of a simple, yet largely unnoticed truth: \emph{when clustering directed networks, directly using eigenvalues and eigenvectors (as opposed to SVD) of the untouched adjacency matrix (as opposed to symmetrized and/or  regularized versions), works very well, especially in harder and sparser regimes}. This statement was hinted in some works, but remained essentially ignored and not backed by theoretical results. In this paper, our goal is to rigorously prove and illustrate this statement.

\paragraph{Spectral clustering of directed networks: overview and related work}\label{sec:RW}

As well-summarized in \cite{von2007tutorial}, spectral algorithms often consist of three steps: (1) a matrix representation of the data, (2) a spectral truncation procedure, and (3) a geometric clustering method on the eigen/singular vectors.

The matrix representation depends on the nature of the data and interaction measurements. Early works were focused on symmetric interactions: the interaction $A_{x,y}$ between two nodes $u,v$ was considered undirected, ie $A_{x,y} = A_{x,y}$. The matrix of interactions $A$ is then hermitian. But in many applications, the interaction are intrinsically \emph{directional}, with $A_{x,y}$ and $A_{y,x}$ not necessarily equal, or, equivalently, the network is \emph{directed}. This covers a wider range of models: buyer/seller networks, ecologic systems with predator-prey interactions,  citations in scientific papers, etc.

 Forerunners in digraph clustering avoided using the interaction matrix $A$; one reason for that is the lack of an orthogonal decomposition for non-normal matrices. Instead, they represented their data with naive symmetrizations of $A$, such as $A+A^*$ or the so-called \emph{co-citation and bibliometric symmetrizations} $AA^*$ and $A^*A$ (which reduces to studying the SVD of $A$), see \cite{satuluri2011symmetrizations} for an overview. Variants of the graph Laplacian were then introduced  (\cite{chung2005laplacians}); they were hermitian but incorporated in some way the directionality of edges and were provable relaxations of normalized-cut problems (see  \cite{leicht2008community} and references in \cite[§4.2-4.3]{malliaros2013clustering}); others used random-walks approaches similar to PageRank (\cite{chen2007directed}, \cite{pentney2005spectral}). It is quite striking that in the survey \cite{malliaros2013clustering}, the authors classify the clustering methods (Chap. 4 therein) \emph{without mentioning the use of the adjacency matrix of directed networks}. Implicit in these early works was the belief that directed networks \emph{need} to be transformed or symmetrized. More recently,
\cite{skew} and \cite{laenen2020higherorder} cleverly introduced $\mathbb{C}$-valued Hermitian matrices.

\medskip

The use of adjacency matrices was advocated later, notably in \cite{li2015analysis}, and several more theoretical works, among them \cite{mariadassou2010uncovering, zhou2019spectral, van2019spectral}. These works can roughly be split in two categories. On one side, the authors of many papers seemed reluctant to use eigenvalues of non-symmetric matrices, and favored the SVD instead, whose theoretical analysis is tractable in some cases. This is notably the case for \cite{sussman2012consistent, mariadassou2010uncovering,zhou2019spectral}. However, as we'll see later, the SVD for non-hermitian matrices suffers the same problem as the eigendecomposition of hermitian matrices: it is sensitive to heterogeneity, and thus less powerful in sparse regimes. On the other side, \cite{li2015analysis, van2019spectral, chen2018asymmetry} are closer in spirit to our work. They explicitly advocate the use of eigenvalues of non-symmetric matrices as a better option for inference problems. The theoretical analysis performed in \cite{li2015analysis, van2019spectral} allows them to guarantee performance in very specific cases, where the underlying graphs have a strong and dense structure (upper-triangular , cyclic or purely acyclic). In a different context (matrix completion), the paper \cite{chen2018asymmetry} was one of the first to prove the efficiency of eigenvalues of non-hermitian matrices in high-dimensional problems.

\medskip
Spectral decompositions of network matrices are generally known to reflect some of the underlying interaction structure between the nodes; this non-rigorous statement has now been mathematically understood in a variety of ways, many of them based on a mathematical model for community networks called the stochastic block-model (\cite{holland1983stochastic, abbe2017community}). Any clustering algorithm can be tested on synthetic data from the SBM to evaluate the reconstruction accuracy, that is, the number of nodes which have been correctly assigned to their community by the algorithm. In this work, we deal with a vast generalization of SBMs, the weighted, inhomogeneous, directed Erd\H{o}s-R\'enyi random graph.  

\begin{definition}\label{def:model}Let $P,W$ be two real $n \times n$ matrices, with $P$ having entries in $[0,1]$. A random weighted graph is defined as follows: the edge set is $V = [n]$; each one of the $n^2$ potential edges $(x,y) \subset V \times V$ is present in the graph with probability $P_{x,y}$ and independently of the others; if present, its weight is $W_{x,y}$. The resulting directed graph will be noted $G=(V,E)$ and its weighted adjacency matrix $A$ is defined by $A_{x,y} \coloneqq W_{x,y}\mathbf{1}_{(x,y) \in E}$. 
\end{definition}
This allows for virtually any structure: classical block-models (assortative or disassortative),  cyclic structures (\cite{van2019spectral}), path-wise structures (\cite{laenen2020higherorder}), overlapping communities (\cite{ding2016overlapping}), bipartite clustering when both sides have the same size (\cite{zhou2019spectral, zhou2018optimal}), contextual information on the edges... 

 In SBMs, a key parameter is the density $d$, the number of edges divided by the size $n$. For inference problems, a lower density means a sparser information. Analyzing the performance of spectral clustering methods can be done using classical perturbation results in regimes where $d$ is large, often of order $n$ (the `dense' regime), see \cite{rohe2011spectral} for instance. However, many real-world networks lie in sparser regimes $d \ll n$, like $d \lesssim \log(n), d \to \infty$ (the `semi-sparse regime') or even $d = O(1)$ (the `sparse regime'), a radically difficult regime in which node degrees are extremely heterogeneous and the graph is not even connected. This behaviour has an impact on spectral quantities when they satisfy Fisher-Courant-Weyl inequalities, like eigenvalues of normal matrices or SVD of non-normal ones, deeply reducing their performance, see \cite{benaych2019largest}. This is why most theoretical works (for both directed and undirected models) were concentrated on $d \sim \log(n)$ regimes (\cite{abbe2020entrywise} among others). In the sparse undirected regime, the celebrated Kesten-Stigum threshold (\cite{BLM}) gives a spectral condition for the emergence of a second eigenvalue $\lambda_2$, beyond the Perron one, in the spectrum of the non-backtracking matrix. The entries of its eigenvector are correlated with the block structure when there are two blocks. In this paper, we describe a general theory of directed Kesten-Stigum-like thresholds for every directed sparse SBM, irregardless of the number of blocks, their size, etc. 

\paragraph{Contributions.}We prove a Master Theorem describing the sharp asymptotics of eigenvalues and eigenvectors of sparse non-symmetric matrices with independent entries, like adjacency matrices of inhomogeneous directed Erd\H{o}s-R\'enyi graphs. This is of independent interest in the field of random matrices. We show how to apply this theorem to the directed SBM and we introduce an elementary community-detection algorithm based on the adjacency matrix. We show why using both left and right eigenvectors is mandatory in sparse regimes. We give numerical and heuristic evidence for why Gaussian Mixture clustering is much more adapted than the popular k-means algorithm. Finally, we illustrate the strength of our method on synthetic data\footnote{The Python software used for the numerical experiments in this paper will be available on a public repository. }.

\paragraph{Notations. }We use the standard Landau notations $o(\cdot), O(\cdot), \sim$; for $k$ integer, $[k]$ stands for $\{1, \dotsc, k\}$. The letters $u,v,w$ will be kept for vectors, the letters $x,y,z$ will be kept for elements of $[n]$ (nodes). We see vectors in $\mathbb{C}^n$ as functions from $[n]$ to $\mathbb{C}$, that is $u = (u(1), \dotsc, u(n))$. The notation $|u|_p$ stands for the $p$-norm of a vector, $|u|_p^p = |u(1)|^p + \dotsc + |u(n)|^p$. We drop the index $p$ iff $p=2$ (euclidean norm). If $M$ is a matrix, $\Vert M \Vert = \sup_{|u|=1}|Mx|$ and $\Vert M \Vert_\infty = \max M$. The Hadamard product of two matrices $A, B$ with the same shape is defined as the entrywise product: $(A \odot B)_{x,y} = A_{x,y}B_{x,y}$. The Frobenius norm is $\Vert M \Vert_F = \sqrt{\sum_{x,y} |M_{x,y}|^2}$. 

\section{The Master Theorem}\label{sec:MT}

Let $P,W$ be two real $n \times n$ matrices, with $P$ having entries in $[0,1]$. The weighted, inhomogeneous directed Erd\H{o}s-R\'enyi model was defined in Definition \ref{def:model}. The weighted adjacency matrix will be noted $A$. We focus on the $n \to \infty$ limit and we suppose in the assumptions thereafter that the graph is sparse, the weights are bounded and the spectral decomposition of $\mathbf{E}[A]$ is not degenerate. Let $Q=\mathbf{E}[A]$ and $K=\mathbf{E}[A\odot A]$ be the first and second entrywise-moments of $A$, given by
\begin{equation}
Q = P \odot W \qquad \text{ and } \qquad K = P \odot W \odot W, 
\end{equation}
in other words $Q_{x,y} = P_{x,y}W_{x,y}$ and $K_{x,y} = P_{x,y}|W_{x,y}|^2$.  Our assumptions are as follows:
\begin{enumerate}
\item $\Vert P \Vert_\infty = O(1/n)$ and $\Vert W \Vert_\infty = O(1)$.
\item The matrix $Q$ has rank $r = O(1)$, is real diagonalizable, and its $r$ eigenvalues $\mu_i$  are well-separated in the sense that there is a constant $c>0$ such that $ |\mu_i - \mu_j|>c, |\mu_j|>c$.
\item the right (resp. left) unit eigenvectors $\varphi_i$ and $\xi_i$ associated with $\mu_i$ are delocalized, in that
\begin{equation}\label{hyp:deloc} | \varphi_i |_\infty, | \xi_i |_\infty = O\left( \frac{1}{\sqrt{n}} \right) \end{equation}
\end{enumerate}
These assumptions will be commented later. We note $\rho = \Vert K \Vert$. The \emph{detection threshold} is 
\[\thresh = \max\left(\sqrt{\rho},\Vert W \Vert_\infty\right) \]
and we note $r_0$ the number of eigenvalues of $Q$ with modulus strictly greater than $\thresh$.

\begin{definition}\label{def:eigendefect}Let $\mu_i$ be an eigenvalue of $Q$ with left and right unit eigenvectors $\varphi_i, \xi_i$. If $|\mu_i|> \thresh$, the (left and right) \textbf{eigendefects} of $\mu_i$ are defined by 
\begin{equation}\label{eq:eigendefects}
R_i = |(K - \mu_i^2 I)^{-1} \varphi_i^2 |_1 \qquad \qquad L_i = |(K^* - \mu_i^2 I)^{-1} \xi_i^2 |_1
\end{equation}
where $\varphi_i^2, \xi_i^2$ are the entrywise squares of $\varphi_i, \xi_i$. 
\end{definition}

These novel quantities will play a role of paramount importance in all this paper and will be commented later. Let us first state our main result, after which we will give some intuition on \eqref{eq:eigendefects}.

\begin{theorem}[Master Theorem]\label{thm:main}Under the above hypotheses, the following holds with probability going to $1$ when $n \to \infty$. The $r_0$ eigenvalues of $A$ with highest modulus, $\lambda_1, \dotsc, \lambda_{r_0}$, are asymptotically equal to the $r_0$ eigenvalues of $Q$ with highest modulus: $|\lambda_i - \mu_i| = o(1)$. All the other $n-r_0$ eigenvalues of $A$ are asymptotically smaller than $\thresh$. Moreover, if $u_i, v_i$ is a left/right pair of unit eigenvectors of $A$ associated with $\lambda_i$, and if $\varphi_i , \xi_i$ is  a left/right pair of unit eigenvectors of $Q$ associated with $\mu_i$, then
  \begin{equation}\label{aligned}
    \left||\langle u_i, \varphi_j\rangle| - \frac{|\langle \varphi_i, \varphi_j\rangle|}{|\mu_i| \sqrt{R_i}} \right| = o(1) \qquad \text{ and }\qquad  \left||\langle v_i, \xi_j\rangle| -\frac{|\langle \xi_i, \xi_j\rangle|}{|\mu_i| \sqrt{L_i}} \right| = o(1).
  \end{equation}
\end{theorem}

\begin{figure}\centering
\begin{tabular}{ccc}
\includegraphics[height=0.06\textheight]{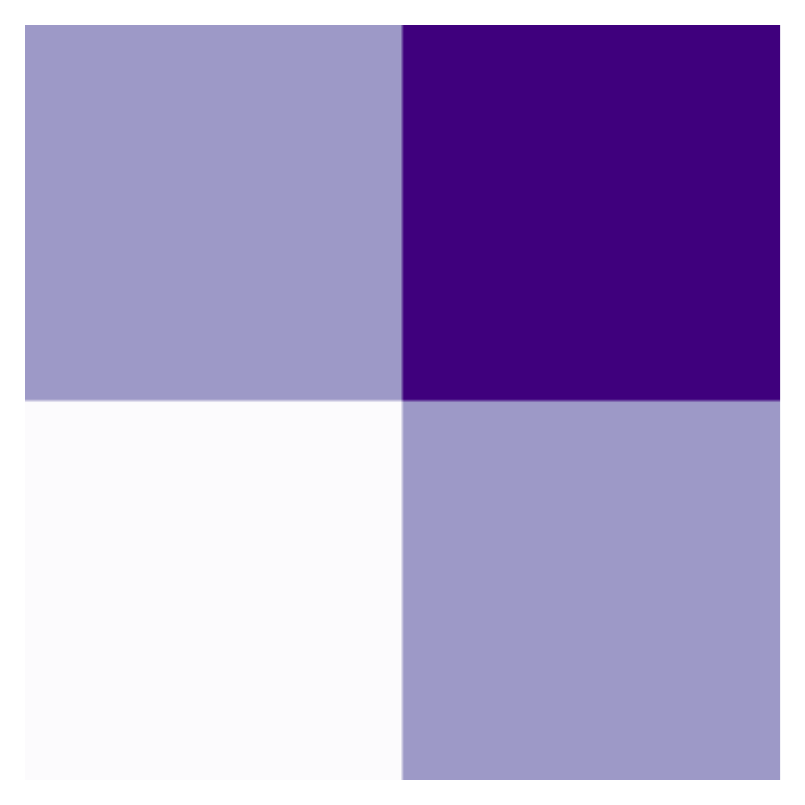}&\includegraphics[height=0.06\textheight]{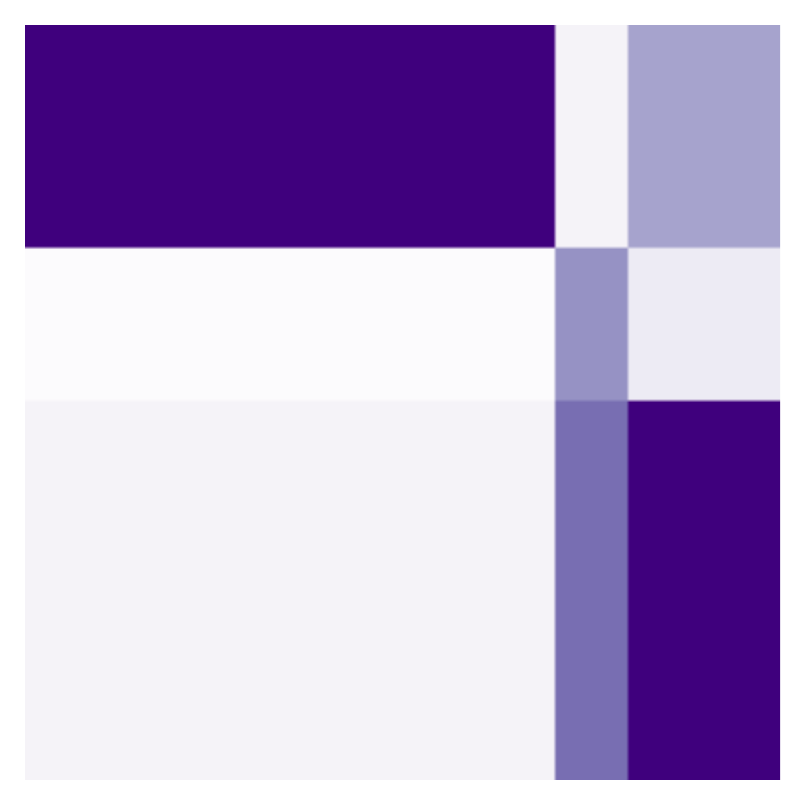}&\includegraphics[height=0.06\textheight]{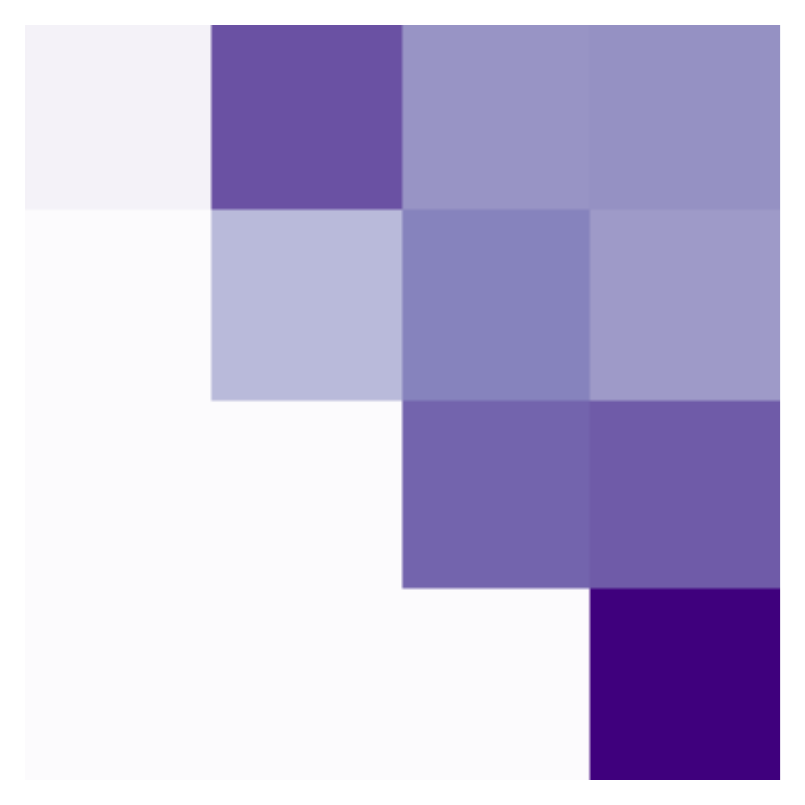}\\
\includegraphics[height=0.12\textheight]{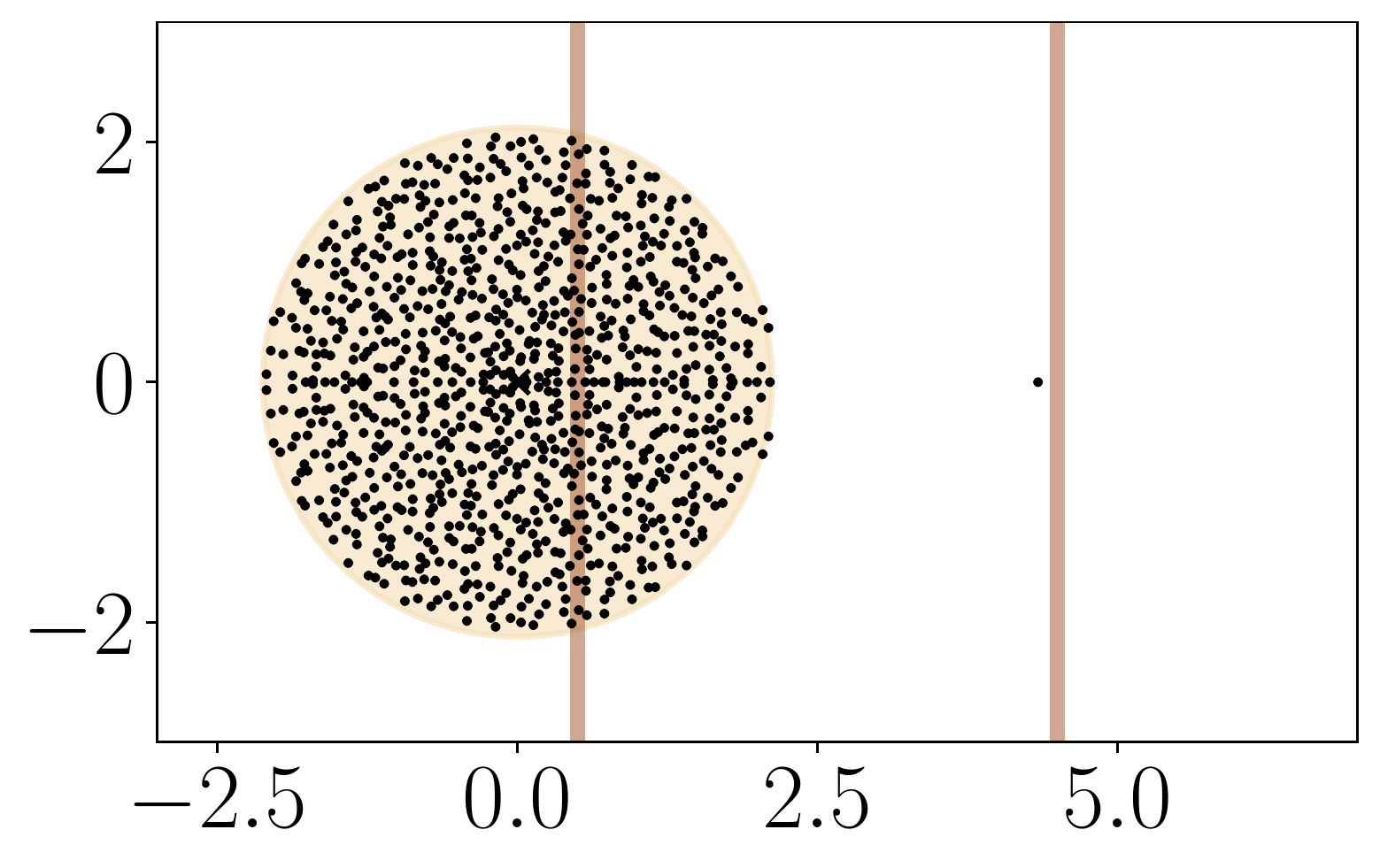}&\includegraphics[height=0.12\textheight]{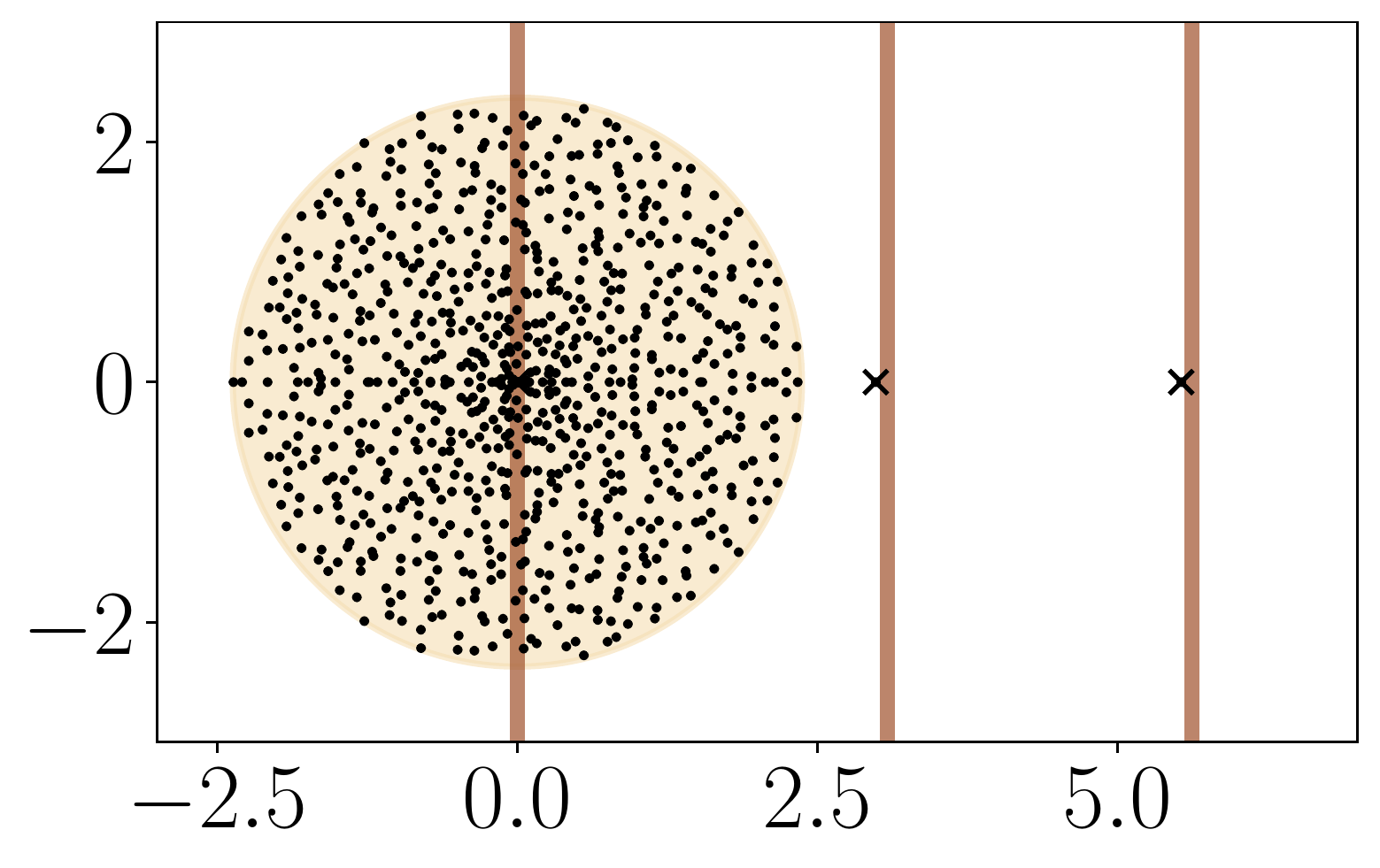}&\includegraphics[height=0.12\textheight]{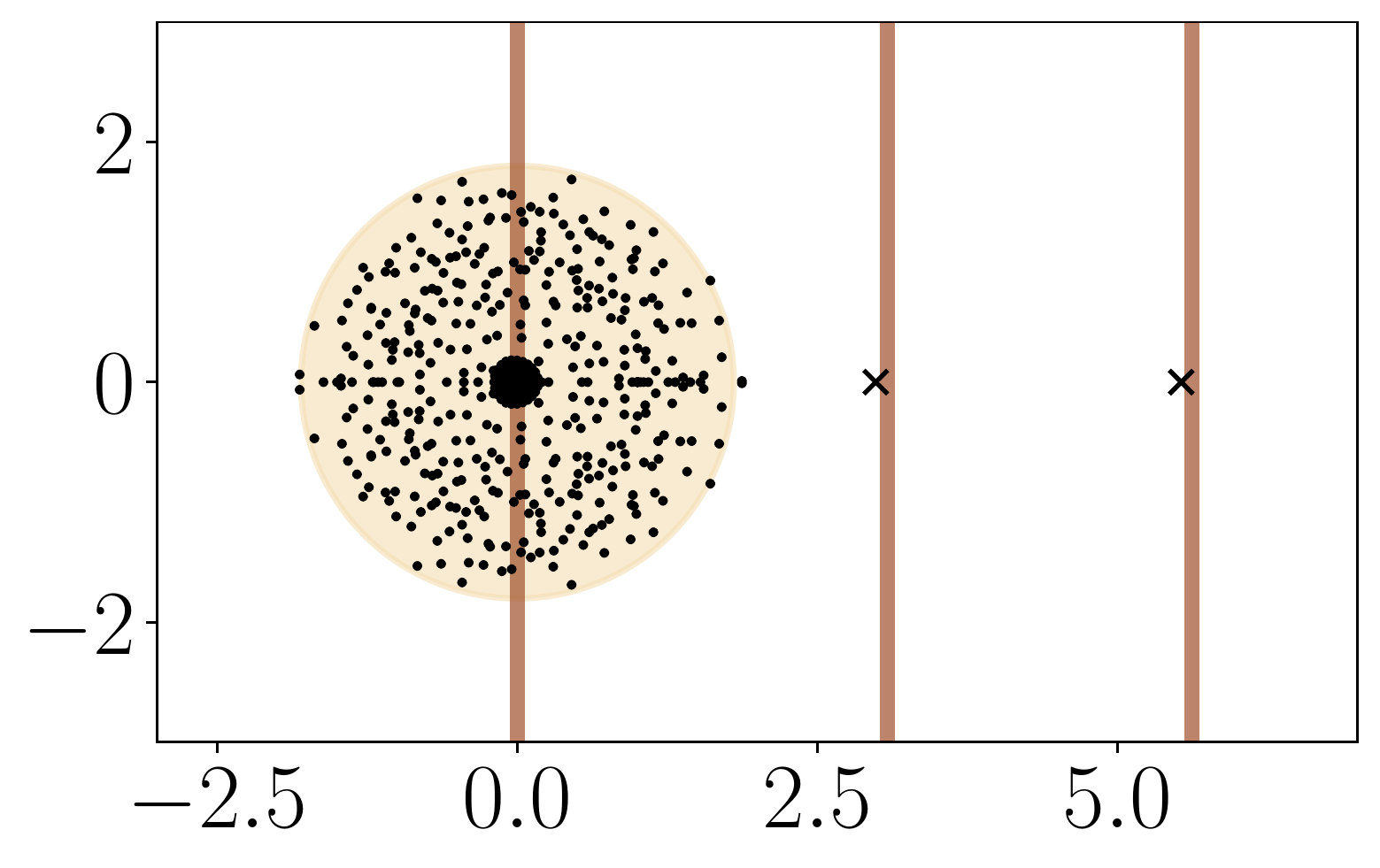}\\
(a) & (b) & (c)
\end{tabular}

\caption{Spectra of inhomogeneous Erd\H{o}s-R\'enyi graphs, with no weights, $n=1000$ nodes. The underlying connectivity matrix $P$ is a block-matrix; the inset of each picture shows a colorplot of $P$, with darker colors indicating higher values. The points are the eigenvalues of $A$. The brown lines indicate the non-zero eigenvalues $\mu_i$ of $Q$. The beige circle behind the eigenvalues has radius $\thresh$. The outliers close to $\mu_i$ for $i \in [r_0]$ are visible for each picture. }\label{fig:MasterTheorem}
\end{figure}

\begin{figure}
\centering
\includegraphics[width = 0.45\textwidth]{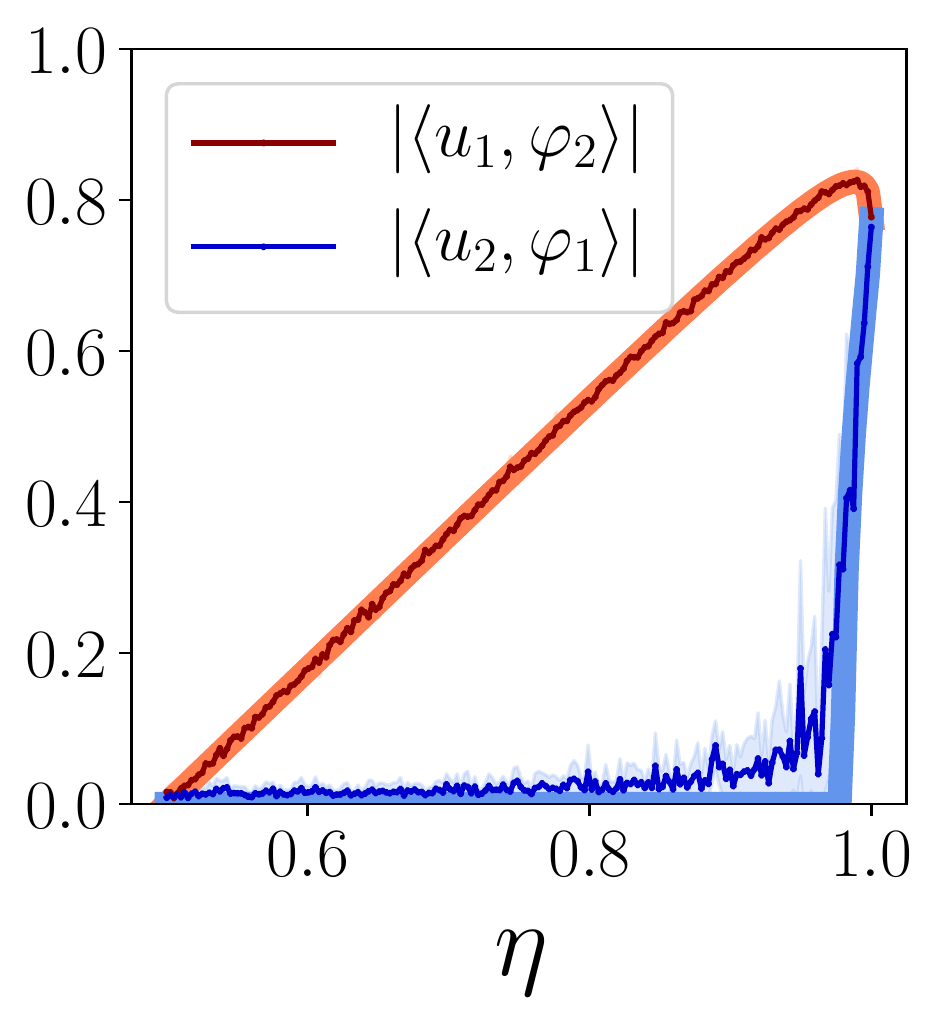}\includegraphics[width = 0.45\textwidth]{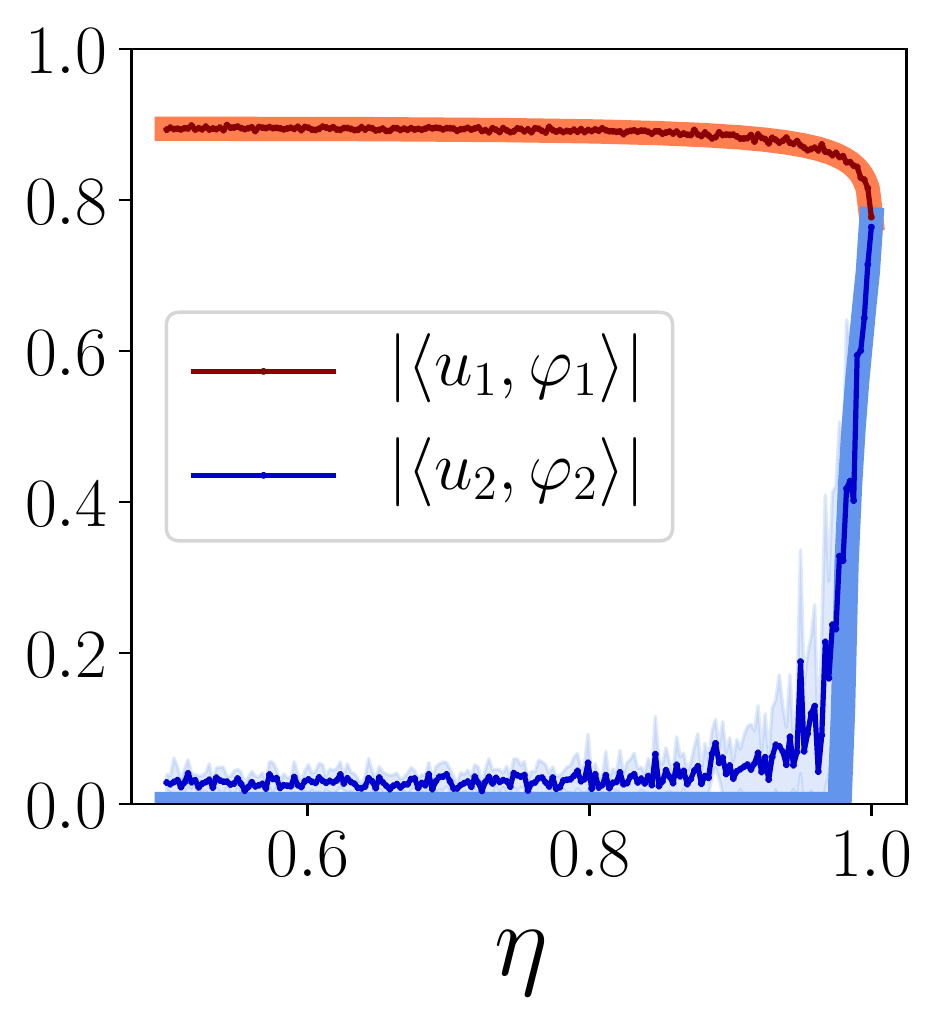}
\caption{\textbf{Bottom}: An illustration of the `right eigenvector part' of the Master Theorem, for the two-block SBM described in Theorem \ref{thm:long} with parameters $s=10$ and $\eta$ ranging from $0.5$ to $1$. For each $\eta$ the results are averaged over $50$ samples (coloured zone is for standard errors). The thin darker lines are $\langle u_i, \varphi_j \rangle$ for $i,j \in \{1, 2\}$. The thick lighter lines are our theoretical predictions $\langle \varphi_i, \varphi_j\rangle/\mu_i \sqrt{R_i}$. The second eigenvector begins to be informative as soon as the second eigenvalue reaches $\thresh$, which happens at around $0.979$ in agreement with our predictions. }
\end{figure}

The eigendefects in Definition \ref{def:eigendefect} measure how much the eigenequations of $Q$ can be `entrywise squared'. For instance, let $\mu$ be an eigenvalue of $Q$ with eigenvector $\varphi$. Then, $(Q - \mu_i I) \varphi_i = 0$. But is $\varphi^2_i$ an eigenvector of $K$ with eigenvalue $\mu_i^2$ ? If $i \in [r_0]$, the answer is obviously no since $\mu_i^2 > \Vert K \Vert$; then, the quantity $1/|(K - \mu_i^2)^{-1}\varphi_i^2|_1$ appearing in the theorem above is a measure of how far $\varphi_i^2$ is from being a $\mu_i^2$-eigenvector of $K$. The theorem says that when $\mu_i$ is gives rise to an outlier $\lambda_i$ in the spectrum of $A$, then the overlap $\langle u_i, \varphi_i\rangle$ between the real eigenvector and the sample eigenvector is higher when $\varphi_i^2$ is far from being an eigenvector of $K$. At a high level, this surprising and new phenomenon comes from an elementary formula regarding the covariance of Poisson sums (see Lemma \ref{lemma:covariance_poisson} in Appendix \ref{app:proof_thm1}), and we conjecture that a similar phenomenon will hold for every random matrix model which is asymptotically Poisson.

\begin{remark}[Comments on the hypotheses. ]\label{rk:hyp}Under \textbf{Hypothesis 1}, if $\Vert P \Vert_\infty \leqslant C/n$, then the expected degree of $x \in [n]$ is $d_x = P_{x,1}+\dotsb + P_{x,n} \leqslant C$ so the average density $d$ is smaller than $C$ (sparse regime). Our proof holds in the semi-sparse regime $d \to \infty, d = n^{o(1)}$, but it is not our primary motivation. In this semi-sparse regime is can be proved that $|\mu_i|\sqrt{R_i}$ goes to $1$, resulting in perfect alignment $\langle u_i, \varphi_i \rangle \to 1$.

Real diagonalizability in \textbf{Hypothesis 2} is here to simplify the proof but the Master Theorem will hold for complex eigendecompositions. The low-rank assumption is standard in the litterature; it can be relaxed by replacing the rank with the effective rank, as in \cite{simon}. The separation assumption is not necessary for eigenvalue asymptotics, but strong separation is necessary for eigenvector overlaps. 

Note that every bound in the hypotheses (such as the bound on $W$, the rank or the delocalization) can be extended to $n^{o(1)}$ at virtually no cost.
\end{remark}

Theorem \ref{thm:main} follows the line of research initiated in \cite{laurent, BLM} and continued by different works, among which \cite{ludovic, simon, pal2019community}. We give an overview of the proof of Theorem \ref{thm:main} in Appendix \ref{app:proof_thm1}.

 \section{Spectral embeddings of directed SBM}\label{subsec:SBM}

\subsection{Model definition}

A powerful aspect of our Master Theorem lies in its application to clustering in directed stochastic blockmodels. In the following, we fix a number of clusters $r$, and a number of vertices $n$, understood to be large. Let $\sigmal, \sigmar: [n] \mapsto [r]$ be the \emph{left} and \emph{right} cluster assignments; that is, a vertex $x$ is said to be in the $i$-th left (resp. right) cluster if $\sigmal(x) = i$ (resp. $\sigmar(x) = i$).
Let $F$ be an arbitrary $r \times r$ matrix with positive entries; the directed SBM is then a random graph $G=([n], E)$ with vertex set $[n]$ and such that for each directed edge $(i, j)$, we have
\[ \mathbb P((i, j) \in E) = \frac{F_{\sigmal(i), \sigmar(j)}}{n}.\]
The aim is then to recover the left (or right) cluster memberships, given an observation of $G$. Note that, for ease of exposition, we did not include weights: $W$ is thus the all-one matrix, and in this setting we have $Q=P = \mathbf{E}[A]$. It is relatively simple to compute the spectral decomposition of $Q$ in this setting:
\begin{proposition}\label{prop:spectral_sbm}
  Let $A$ be the adjacency matrix of $G$. The non-zero eigenvalues of $P = \mathbf{E}[A]$ are exactly those of the modularity matrix $F\Pi$, where
  \[ \Pi_{ij} = \frac{\left|\sigmal^{-1}(j) \cap \sigmar^{-1}(i)\right|}{n}. \]
  Additionnally, the associated left eigenvectors of $P$ are constant on the right clusters, while the right eigenvectors are constant on the left clusters.
\end{proposition}

Therefore, as per our main theore, the left/right eigenvectors of the adjacency matrix of $G$ are close to their expectation, which is constant on the right/left clusters. We thus expect a clustering algorithm on those eigenvectors to be able to recover at least a fraction of the community memberships. We refer to Appendix \ref{app:SBM} for a more complete spectral analysis of the matrix $P$, as well as a formulation of Theorem \ref{thm:main} suited to the SBM setting.

\begin{remark}
  Whenever $\sigmal = \sigmar$, as is often the case, both left and right eigenvectors are constant on the clusters; this effectively doubles the signal to recover $\sigmar$.
\end{remark}

An important question to ask is the following: how many eigenvectors of $A$ do we need to be able to reconstruct the clusters ? It is often assumed that $r$ eigenvectors are needed to recover the memberships (see e.g. \cite{von2007tutorial}). However, in our DSBM setting, we propose the following heuristic:
\begin{center}
   \emph{Partial cluster recovery is possible as soon as the first $r_0$ eigenvectors of $\mathbf E [A]$ \\ are sufficient to recover the clusters.}
\end{center}
Here, $r_0$ is the same as in Theorem \ref{thm:main}, and denotes the number of eigenvalues of $\mathbf E[A]$ that get reflected in the spectrum of $A$. Since we showed that the eigenvectors of $\mathbf E[A]$ are constant on the clusters, this is equivalent to the function $k \mapsto (\psi_1(k), \dots, \psi_{r_0}(k))$ being injective, where the $\psi_i$ are the right (resp. left) eigenvectors of $F\Pi$ (resp. $\Pi F$). This can happen when $r_0 \ll r$, and even in some cases when $r_0 = 1$, which is a huge improvement on the threshold for reconstruction. Additional eigenvectors may of course increase the recovery accuracy; however, in some cases, the additional information they bring is nullified by the increase in dimensions for the clustering algorithms.

\subsection{SBM with a pathwise structure}\label{subsec:cycle}
\paragraph{General case. }

We restrict to the classical SBM described earlier, with a specific shape known as \emph{pathwise structure}, and notably studied in \cite{laenen2020higherorder}. It is a good model for flow data. In this model, we have $\sigmar = \sigmal$, and the clusters partition $[n]$ in $r$ parts of equal size. The underlying $r \times r$ connectivity $F$ is given by
\begin{equation}\label{def:Fcyclic}
  F = s \begin{pmatrix}
    1/2     & \eta                     \\
    1- \eta & 1/2    & \eta            \\ \\
            & \ddots & \ddots & \ddots \\ \\ 
            &        &1-\eta & 1/2 & \eta \\ 
            &        &             & 1-\eta & 1/2
  \end{pmatrix}
\end{equation}
where $s>1$ is the density parameter and $\eta \in [1/2, 1]$. The modularity matrix is therefore given by $F/r$. The matrix $F$ shown in \eqref{def:Fcyclic} is a tridiagonal Toeplitz matrix; such matrices have been extensively studied and their eigendecomposition is known (see Appendix \ref{app:tridiag}): as a result, cluster recovery is possible as soon as the top eigenvalue of $F/r$ is at least one. This happens in particular whenever $s \geq 2r$.

\paragraph{Two blocks: explicit computations.}\label{sec:two-blocks}

In the case of two blocks $r=2$ with the same size $n/2$, the connectivity matrix $F$ is equal to 
\begin{equation}\label{def:F}
F = \begin{pmatrix}
s/2 & s\eta \\ s(1-\eta) & s/2
\end{pmatrix}
\end{equation}
Define the parameter $\theta = 2\sqrt{\eta(1 - \eta)}$. The spectral structure of $F$ is described in the following lemma:
\begin{lemma}
  The two eigenvalues of $F$ are $\upsilon_1 = s\frac{1 + \theta}2$ and $ \upsilon_2 = s\frac{1 - \theta}2$, with corresponding unit right eigenvectors $f_i$ and unit left eigenvectors $f_i$ given by
  \begin{align*}
    &f_1 = (\sqrt{\eta}, \sqrt{1-\eta}), &
    &g_1 =(\sqrt{1-\eta}, \sqrt{\eta}), \\
    &f_2 = (\sqrt{\eta}, -\sqrt{1-\eta}), &
    &g_2 =(\sqrt{1-\eta}, -\sqrt{\eta}).
    \end{align*}
\end{lemma}
The eigenvectors of $P$ thus verify
\[ \varphi_i(x) \propto f_i(\sigma(x)) \qquad \text{and} \qquad \xi_i(x) \propto g_i(\sigma(x)), \]
and Theorem \ref{thm:main} applies in this setting:

% As a result, $|\nu_2|> \thresh$ (or equivalently $r_0=2$) if and only if $\eta > \eta(s)$, where 
% \begin{equation}\label{thresh:s}
% \eta(s) \coloneqq   \frac{1+\sqrt{1 - \left( 1 + \frac{2 - 2\sqrt{2s+1}}{s} \right)^2}}{2}.
% \end{equation}
% This function is depicted in the bottom right panel of Figure \ref{fig:thresholds}. Note that $\eta(s)$ converges slowly to $1/2$; for example $\eta(10) \approx 0.944$, which is still close to 1.

\begin{theorem}\label{thm:long}Under the above assumptions, with high probability the following holds.
 
 \bigskip
 
\textbf{1) }If $s<4(1+\theta)/(1 - \theta)^2$, then $r_0=1$. The Perron eigenvalue of $A$, namely $\lambda_1$, is asymptotically equal to $\upsilon_1 / 2$, and all the other eigenvalues have modulus asymptotically smaller than $\sqrt{\upsilon_1 / 2}$. Moreover, if $u_1, v_1$ is a left/right pair of unit eigenvectors associated with $\lambda_1$, then $\lim_{n \to \infty} |\langle u_1, \varphi_1 \rangle | =\lim_{n \to \infty} |\langle v_1, \xi_1 \rangle | =
a_{1,1}$ where $a_{1,1}$ is a completely explicit function of $s, \eta$ that satisfies
\[ a_{1, 1} = 1 - \frac2s \cdot \frac{1 + \theta^2}{(1+\theta)^2} + O\left(\frac1{s^2}\right) \]

\bigskip

\textbf{2) }If instead $s<4(1+\theta)/(1 - \theta)^2$, then $r_0=2$. The Perron eigenvalue of $A$, namely $\lambda_1$, is asymptotically equal to $s \upsilon_1 / 2$, the second eigenvalue $\lambda_2$ is asymptotically equal to $\upsilon_2 / 2$ and all the other eigenvalues have modulus asymptotically smaller than $\sqrt{\upsilon_1 / 2}$.

Moreover, if $u_i, v_i$ is a left/right pair of unit eigenvectors associated with $\lambda_i$ for $i = 1, 2$, then $\lim_{n \to \infty} |\langle u_1, \varphi_1 \rangle | =\lim_{n \to \infty} |\langle v_1, \xi_1 \rangle | =
a_{1,1}$ as above, and additionnally $\lim_{n \to \infty} |\langle u_2, \varphi_2 \rangle | =\lim_{n \to \infty} |\langle v_2, \xi_2 \rangle | =
a_{2,2}$ with $a_{2,2}$ a completely explicit function of $s,\eta$ that satifies
\[a_{2, 2} = 1 - \frac2s \cdot \frac{1 + \theta^2}{(1-\theta)^2} + O\left(\frac1{s^2}\right).\]
\end{theorem}

The threshold $s > 4(1+\theta)/(1 - \theta)^2$ can also be rewritten as $\eta > \eta(s)$, where $\eta$ is an explicit function of $s$ (see equation \eqref{eq:etas_expression} in the Appendix). This will be the preferred formulation through the rest of the paper. Note that $\eta(s)$ decreases to $1/2$ quite slowly: as an example, we have $\eta(10) \approx 0.979$ and $\eta(50) \approx 0.885$, and it can be shown that $\eta(s) \sim cs^{-1/4}$.

\section{Geometric clustering and community detection}

Our Master Theorem describes the information given by the spectral embedding $\mathcal X$ on the underlying model. Most spectral clustering pipelines then perform geometric clustering based on $\mathcal X$.

\subsection{Algorithm and measure of performance}
Our algorithm computes the left and right eigenvectors $x_i, y_i$ associated with the $r_0$ largest eigenvalues of $M$, then defines an embedding of the nodes of $[n]$ in $\mathbb{R}^{2k}$ by setting 
\begin{equation}\label{spectral_embedding}\mathcal{X}^A(x) = (u_1(x), \dotsc, u_{r_0}(x), v_1(x), \dotsc, v_{r_0}(x)). \end{equation}
Then, we cluster these $n$ points using the Gaussian Mixture Model for clustering (\cite{mclachlan1988mixture}). We insist on the fact that no data preprocessing is needed: no high-degree trimming, no pruning, no normalization.
 The complexity of our algorithm is similar to all the spectral clustering procedures: it needs the computation of at most $2r$ left/right eigenvectors where $r$ is generally $O(\log(n)^c)$, and then doing a clustering method with at most $r$ clusters on a $n \times 2r_0$ embedding. 

\begin{remark}
The number $r_0$ is a priori problem-dependent. However, since $r_0 \leqslant r$ and $r$ is low in most problems, one can loop over $r_0$ at a minor cost. The Master Theorem allows for a more reasonnable possibility, which is to directly estimate $r_0$ from the data as the number of isolated eigenvalues outside the bulk of the spectrum. This can easily be done either by visual inspection (see Figure \ref{fig:MasterTheorem}) or by some ad hoc statistical rule and it does not require a priori knowledge of $r$ --- unlike many methods in the litterature. 
\end{remark}

\begin{algorithm}[t]
\centering
\begin{algorithmic}[1]
\STATE \textbf{Data}: a $n\times n$ adjacency matrix $A$; a number of clusters $k$; a rank $r_0$.
\STATE Compute  the $r_0$ largest eigenvalues of $M$ and their unit left and right eigenvectors $u_i, v_i$. \;	 
\STATE Define the spectral embedding $\mathcal{X}^A = \{\mathcal{X}_x^A : x \in [n]\}$ as in \eqref{spectral_embedding}.\;
\STATE Apply a GMM-clustering method on the cloud $\mathcal{X}^M$.
\RETURN The partition of $[n]$ based on the output of GMM-clustering.
\end{algorithmic}

\caption{Spectral clustering of $n$ nodes, based on the adjacency matrix $A$. \label{alg:spectral}}
\end{algorithm}

 In the stochastic block-model, we have a notion of ground-truth clustering $\sigma : V \to [k]$, where $\sigma(x)=i$ denotes the membership of node $x$ to the $i$-th cluster. If our procedure outputs a clustering $\hat{\sigma}$, the performance of this clustering is measured through the overlap, also called Rand Index: it is the proportion of pairs of nodes on which $\sigma$ and $\hat{\sigma}$ agree on membership, that is
\begin{equation}
\ov(\sigma, \hat{\sigma}) = \frac{1}{\binom{n}{2}}\sum_{\{x,y\}} \mathbf{1}\{\sigma \text{ and } \hat{\sigma} \text{ agree on the edge } \{x,y\}\}, 
\end{equation}
where `agree' means that either $\sigma(x)=\sigma(y)$ and $\hat{\sigma}(x) = \hat{\sigma}(y)$, or $\sigma(x)\neq \sigma(y)$ and $\hat{\sigma}(x) \neq \hat{\sigma}(y)$. 
Without any information on $\sigma$, a blind guess for $\sigma(x)$ is to randomly assign $x$ to one of the $k$ clusters. This is called a dummy label, $\hat{\sigma}_{\mathrm{dummy}}$. The adjusted overlap is often preferred to the former:
\begin{equation}\label{def:aov}
\mathrm{aov}(\sigma, \hat{\sigma}) =\frac{\mathrm{ov}(\sigma, \hat{\sigma}) - \mathbf{E}[\mathrm{ov}(\sigma, \hat{\sigma}_{\mathrm{dummy}})]}{1 - \mathbf{E}[ \mathrm{ov}(\sigma, \hat{\sigma}_{\mathrm{dummy}})]}.
\end{equation}
An adjusted overlap of $1$ indicates a perfect recovery of $\sigma$ (up to permutation), while an overlap of $0$ indicates that $\hat{\sigma}$ is not better than a dummy guess at recovering $\sigma$. In the litterature this is often called Adjusted Rand Index. It will be our measure of performance in our numerical experiments.

\begin{remark}[notions of consistency]\label{rk:consistency}
\emph{Strong consistency} of a procedure corresponds to $\mathrm{aov} = 1$, that is: all the labels are correctly guessed. Weak consistency is when $\lim \mathrm{aov} = 1$ when $n \to \infty$. \emph{Partial consistency} is when $\liminf \mathrm{aov} >0$, meaning that the algorithm does strictly better than random guess, a task called \emph{detection}. In the undirected setting, strong consistency is feasible in the regime $d / \log(n) \to \infty$ (\cite{abbe2020entrywise} and weak consistency as long as $d \to \infty$ (\cite{gao2017achieving} and the survey section therein). Note that in the sparse regime with $d$ constant (this paper), even weak consistency is not achievable because of a constant proportion of isolated nodes. Partial consistency in the undirected case was achieved in \cite{BLM, mossel2015reconstruction} under a spectral condition on the modularity matrix. We expect similar results to the ones above to hold in the undirected case.
\end{remark}

\subsection{Choice of the clustering algorithm}

When it comes to the last step of spectral clustering, i.e. geometric clustering on the spectral embedding, the overwhelming choice of algorithm is the $k$-means (see for example \cite{van2019spectral,zhou2019spectral}). It is very simple to implement, and its tractability allows for explicit performance bounds for spectral clustering \cite[Theorem 2]{skew}. However, our numerical experiments (see Appendix \ref{app:gaussianclustering}) show an interesting phenomenon: the addition of a second informative eigenvector appears to decrease the algorithm's performance ! Our assumption is that the increase in dimensions for clustering caused by the introduction of the second eigenvector nullifies the additional information it brings. To the contrary, our experiments showed a significative increase in performance when using GMMs, although very little is known in terms of their theoretical footing. A mix of theoretical and empirical results allow us to present a simple explanation: the eigenvectors of $G$ are indeed close to a mixture of Gaussian distributions.

\begin{theorem}\label{thm:eigenvector_convergence}
  Consider the SBM as described in Section \ref{subsec:SBM}, with $\sigmar = \sigmal$. Assume that $\mu_i$ is an isolated informative eigenvalue of $P$, and let $u_i$ be a right eigenvector of the adjacency matrix of $G$ such that $|u_i| = \sqrt{n}$.
  Then we have the following convergence in distribution: for all $j \in [r]$,
  \[ \frac{1}{\card(\sigma^{-1}(j))} \sum_{\sigma(x)= j} \delta_{u_i(x)} \xrightarrow[n \to +\infty]{d} Z_{i,j}, \]
  where $Z_{ij}$ is a random variable with known mean $\mu_{ij}$ and variance $\sigma_{ij}^2$ (see equation \eqref{eq:muij_sigmaij} in the Appendix). Similar results hold for the right eigenvector $v_i$.
\end{theorem}
In the above theorem, $\delta$ denotes the Dirac delta; the LHS is therefore simply the discrete measure on the entries of $u_i$ on the $j$-th community. The proof of this theorem, as well as an explicit derivation of $\mu_{ij}$ and $\sigma_ij$, can be found in the Appendix. It implies, in particular, that the distribution of $u_i$ can be seen as a mixture of $r$ different distributions.

We do not claim (and it is indeed false, see Remark \ref{rk:nongaussien}) that $Z_{i,j} \sim \mathscr N(\mu_{ij}, \sigma_{ij}^2)$; however, numerical experiments performed in Figure \ref{fig:fluctuations} appear to show that, at least when the mean degree of the graph is large, the distribution of $u_i$ behaves as a mixture of Gaussian distributions. This brings us to the conjecture:

\begin{conjecture}\label{conjecture}
  When the top eigenvalue of $P$ goes to infinity, the distribution of $Z_{ij}$ approaches that of a normal random variable with same mean and variance.
\end{conjecture}

If proven true, this conjecture would give theoretical footing to the performance of GMMs, whose edge over other algorithms is only observed empirically for now.

\begin{figure}
\centering
\begin{tabular}{c}
\includegraphics[width = 0.88\textwidth]{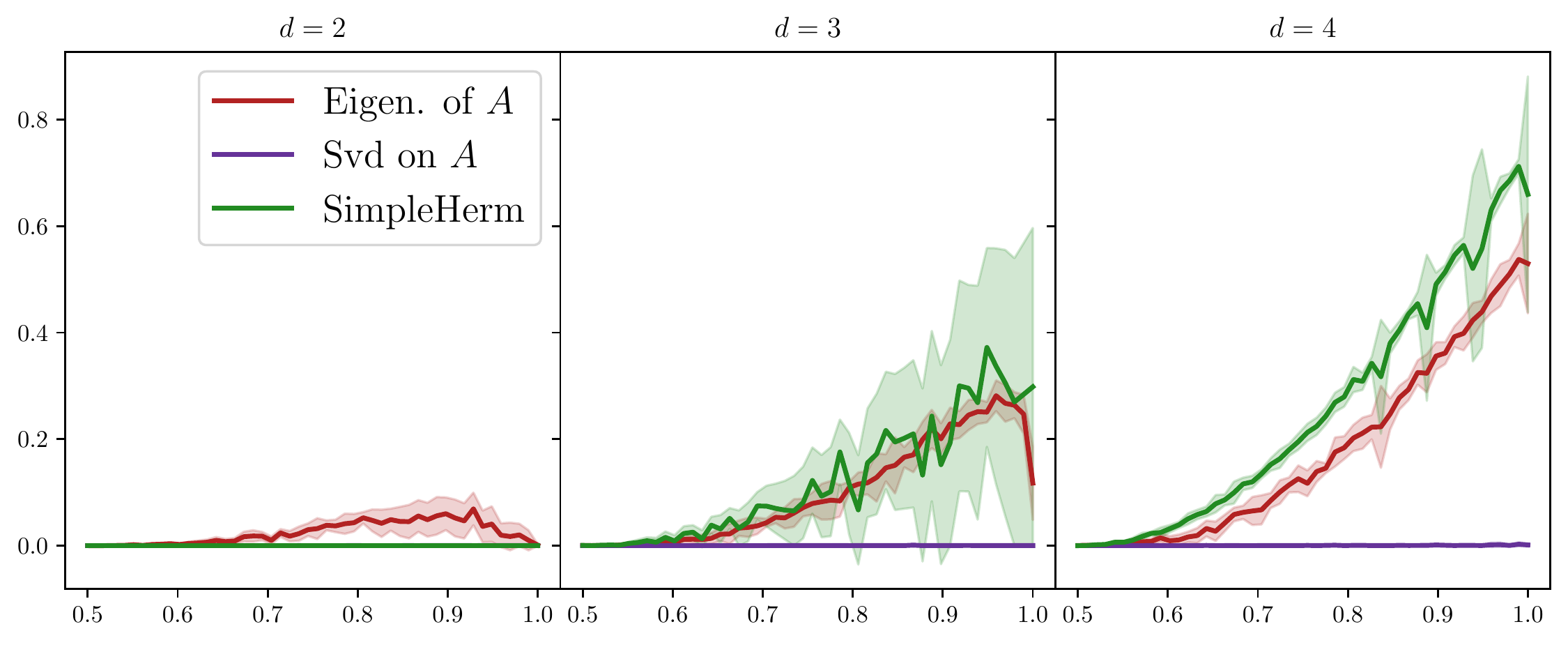}\\
(a) $k=2$ blocks\\
\includegraphics[width = 0.88\textwidth]{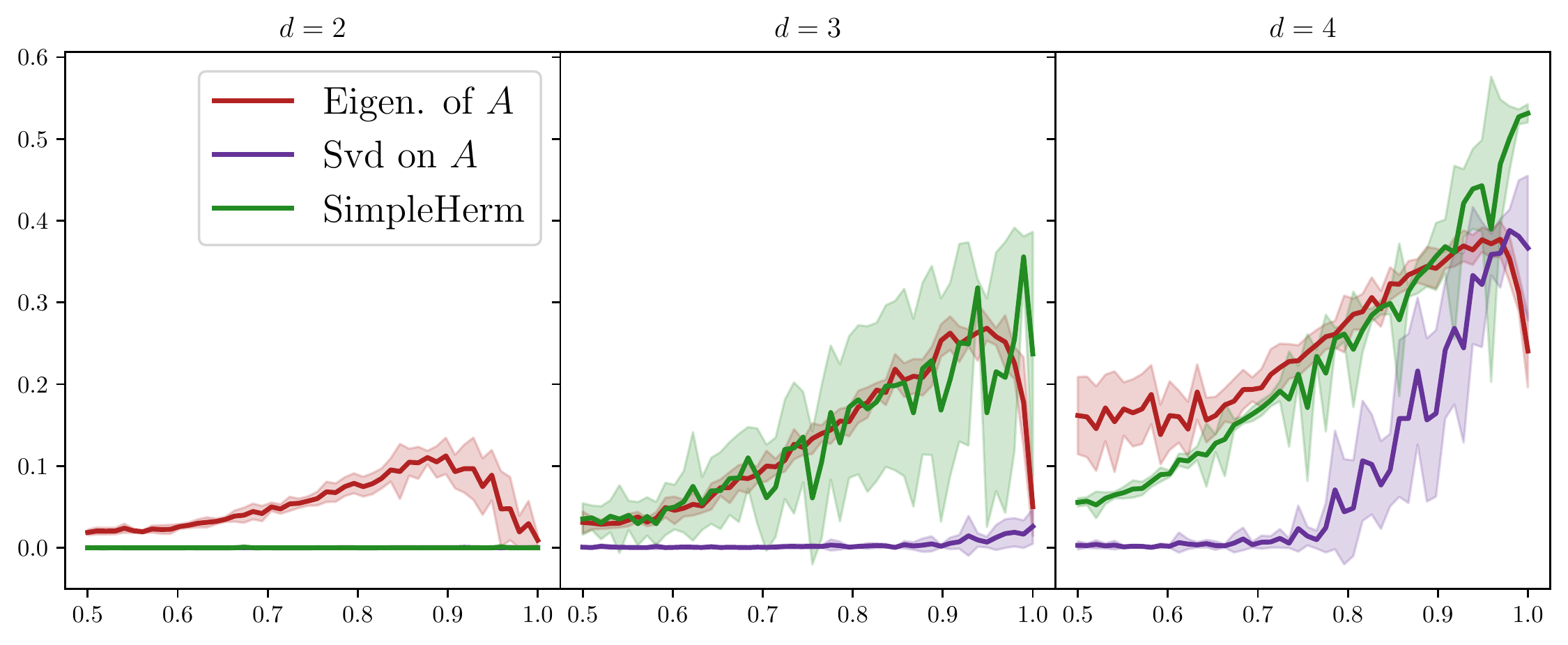}\\
(b) $k=4$ blocks\\
\includegraphics[width = 0.88\textwidth]{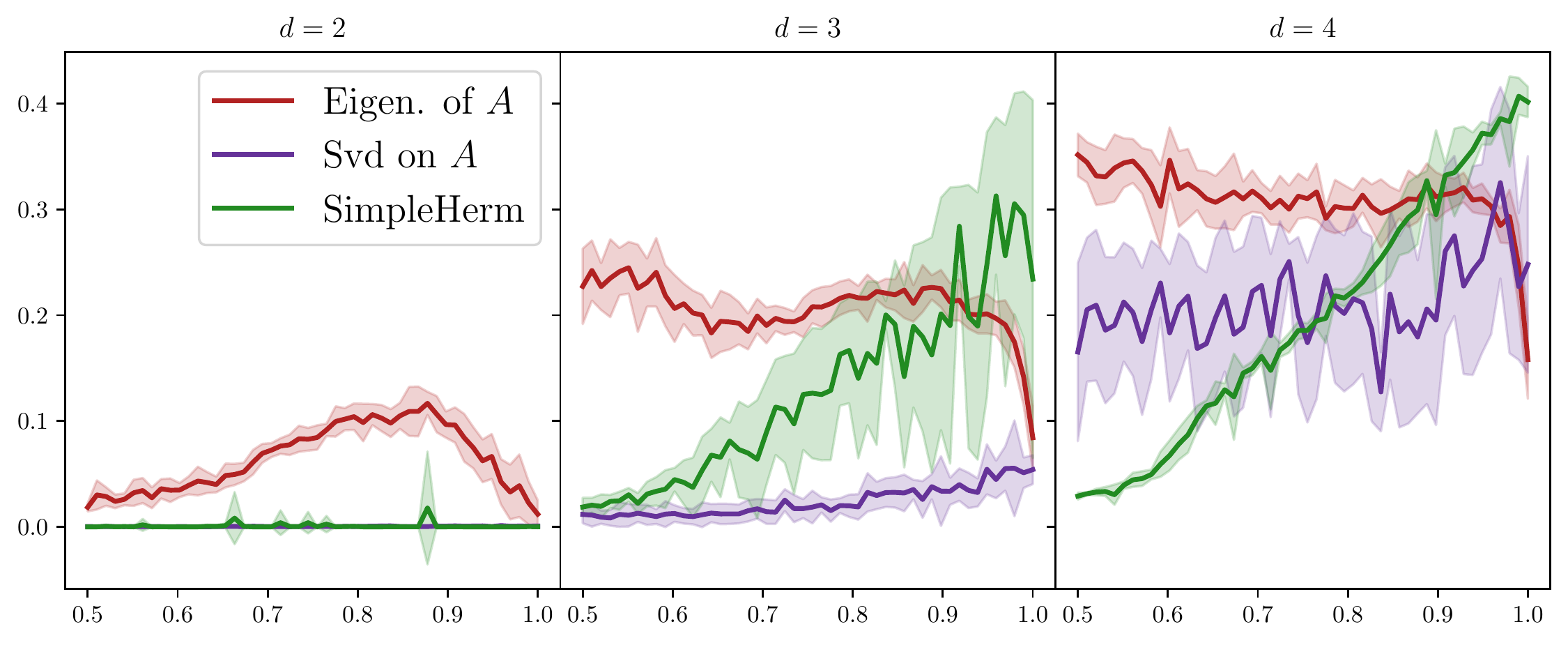}\\
(c) $k=6$ blocks
\end{tabular}

\caption{Averages of labels reconstruction scores (adjusted overlaps), averaged over 20 runs for 50 values of $\eta$ ranging from $0.5$ to $1$, and for different mean degree $d$. The number of nodes is $n=2500$. Coloured zones are for standard errors of the corresponding method. The parameter $s$ corresponding to the various values $d$ are given in Table \ref{tab:special_numbers} page \pageref{tab:special_numbers}).
}\label{fig:results}
\end{figure}

\subsection{Numerical validation of our results}

\paragraph{Tested methods.}We compare our Algorithm \ref{alg:spectral} with two other methods for digraph clustering. Both methods end with a k-means clustering on a spectral embedding. The first method uses the $k$ left and $k$ right top singular vectors of the adjacency matrix, where $k$ is the number of blocks. The second one is SimpleHerm: we define a complex Hermitian matrix by $H_{x,y} = \omega A_{x,y} + \overline{\omega}A_{v,u}$ where $\omega$ is the $\lceil 2\pi k\rceil$-th root of unity.We then use the eigenvector of the smallest eigenvalue $\lambda_1$ of $L = I - D^{-1/2}HD^{-1/2}$ with $D$ the diagonal degree matrix ($D_{x,x} = d^+_x + d^-_x$); since its entries are complex, it is viewed as an embedding on $\mathbb{R}^2$. This method was introduced in \cite{laenen2020higherorder}, and was convincingly shown to outperform other classical methods in semi-sparse regimes.

SVD and our method are agnostic to problem structure, but SimpleHerm is well-fitted to flow networks. The performance guarantees of our method relies on the probabilistic properties of the SBM, while SimpleHerm satisfies deterministic Cheeger-like inequalities (see \cite{laenen2020higherorder}). It would be interesting to test these methods on more general models of directed networks. 

\paragraph{Setting.}For $50$ values of $\eta$ equally spread between $0.5$ and $1$ we sampled $20$ directed SBMs with connectivity matrix $F$ as in \eqref{def:Fcyclic} and with $n=2500$ nodes. The parameters are $k$, the number of blocks and equal to $2,4,6$ (the blocks have the same size $n/k$). For the parameter $s$, we chose the unique $s(d,k)$ so that the the mean degree of our model with $k$ blocks (given in the formula \eqref{pathwise:meandegree}) is equal to $d=2, 3$ or $4$, see Table \ref{tab:special_numbers} at page \pageref{tab:special_numbers} and discussion therein. We insist on the fact that the mean degree in our model is extremely low, and in particular stays under the $\log(n) = \log(2500) \approx 7.82$ barrier. Our performance measure is the adjusted overlap \eqref{def:aov}, between the labelling output by the tested method, and the true labelling. 

\paragraph{Results.}The results of our experiments\footnote{In a preliminary version of this paper, the method SimpleHerm was incorrectly implemented. } are in Figure \ref{fig:results}. With extremely low degrees ($d=2$), our method (red curve) is the only one to catch a signal, the two other ones are unable to detect any community structure. For slightly higher $d=3, 4$, our method seems to globally compare with SimpleHerm and be superior to SVD. When the asymmetry is closer to $1/2$ (hard regime), our method performs better, see for instance the very neat advantage for $k=6$ blocks, where our method reaches more than 20\% overlap for $\eta\approx 1/2$, against no detection at all for the other methods. When $\eta$ approaches $1$, the performance of our method collapses back to low overlaps, while SimpleHerm has very high performances. We expect this phenomenon to be caused by the fact that when $\eta \to 1$, the eigenvectors of $F$ all align with $(1, 0, \dotsc, 0)$. On a side note, we remark that our method has a better precision, with the standard error (coloured zones) being generally smaller.

\section{Conclusion and future prospects}

We rigorously described the behaviour of a simple spectral embeddings, using the eigenvectors of non-symmetric matrices, and we numerically show that our algorithm using Gaussian Mixture clustering has suprisingly good results against state-of-the art methods in digraph clustering, especially in the difficult regime where the model density is $O(1)$. 
The main weakness of our theory is that it does not apply to rectangular matrices directly, but the randomsplit-squaring strategy as in \cite{simon} is directly applicable here.
Since our theory is new, we chose to keep the exposition as simple and general as possible, but many new directions seem to be promising: among them is the possibility to use the distance-matrix $A^{(\ell)}_{x,y} = \mathbf{1}_{d^+(x,y) \leqslant \ell}$ instead of $A$, which should result in a method which is more robust to adversarial perturbations, as in \cite{abbe2020graph, stephan2019robustness}. Regarding the gaussianity of the model, we conjecture that the fluctuations of the eigenvalues are Gaussian in the sparse regime; as mentioned, the fluctuations of the eigenvector entries will not be Gaussian, but we now explore a proof of the convergence of these fluctuations when the density of the model increases. 

\newpage

\bibliographystyle{plainnat}
\bibliography{bibli}

\newpage
\appendix

\section{Gaussian Mixture clustering and Gaussian fluctuations}\label{app:gaussianclustering}

\begin{figure}\centering

\includegraphics[width = 0.9\textwidth]{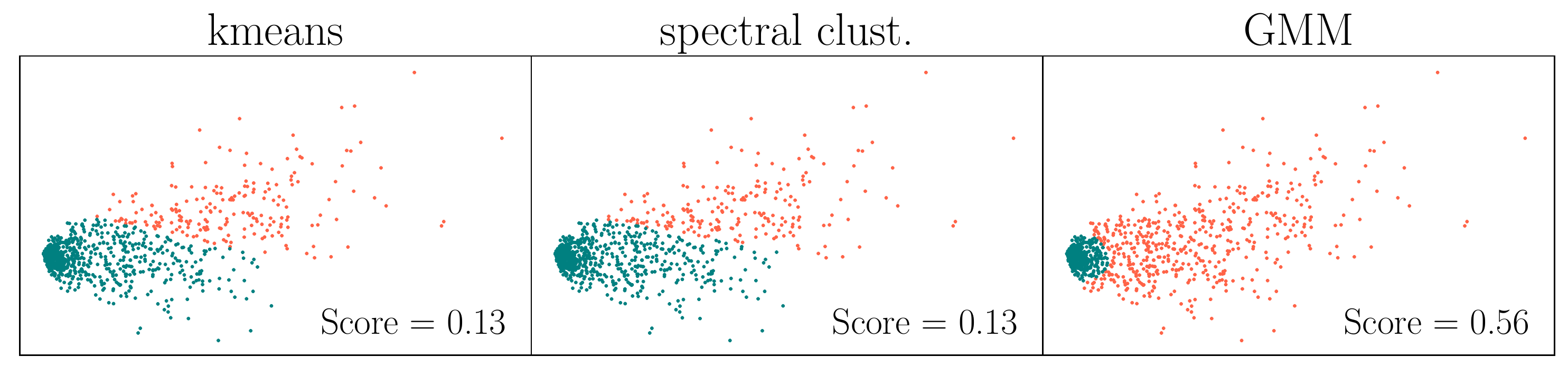}

\includegraphics[width = 0.9\textwidth]
{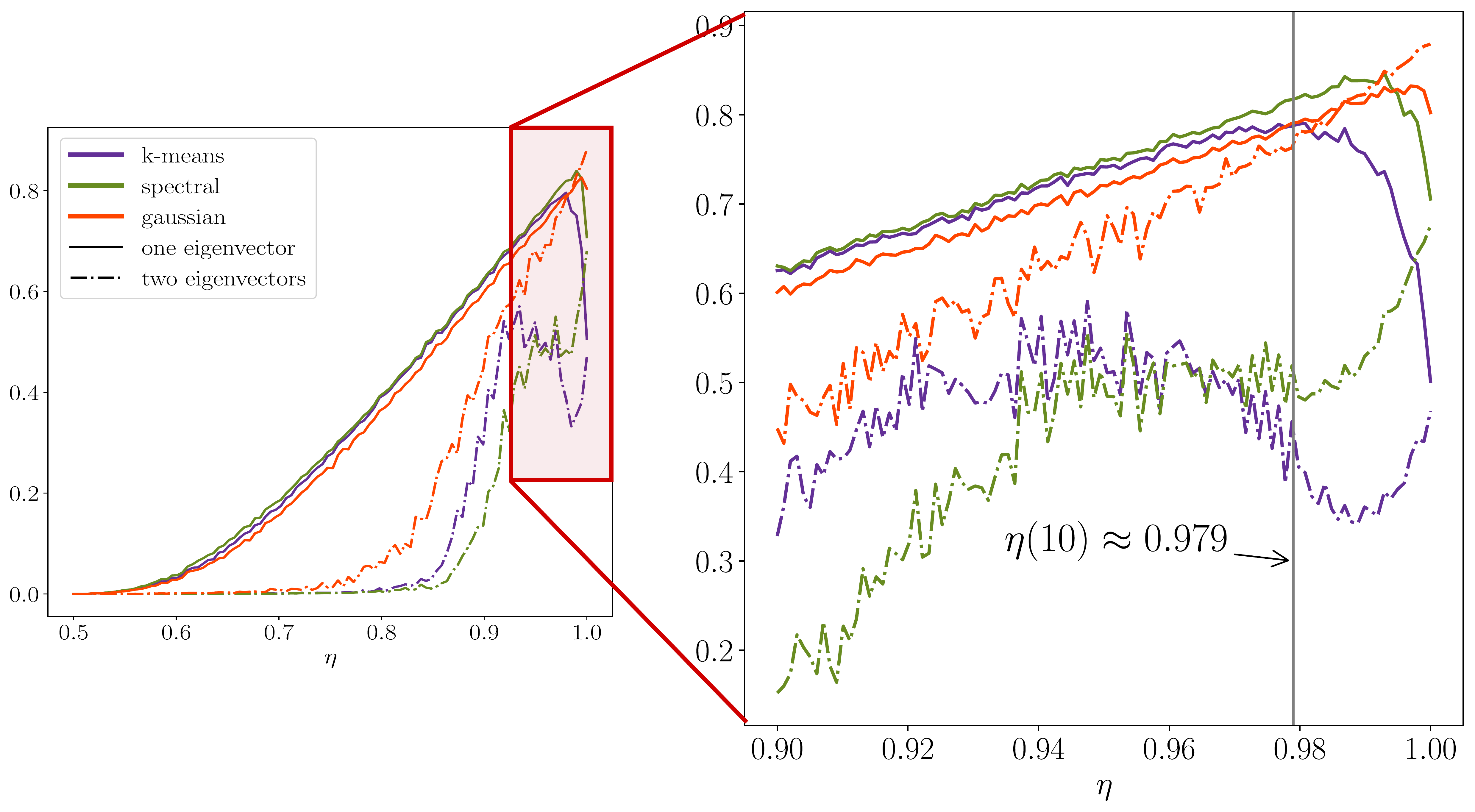}

\caption{\textbf{Top}. Here $n=2000$, $s=10$ and $\eta = .99 > \eta(10)=$, thus $r_0=2$. The Score is the adjusted overlap in \eqref{def:aov}. \newline \textbf{Middle}. For the two-block model with $\eta \in [0.5, 1]$ we plotted the average Adjusted overlap over 100 runs of several clustering methods on spectral embeddings using either the embedding with the Perron vector $x \mapsto (u_1(x), 0)$ (solid lines) or the embedding with two dominant eigenvectors $x \mapsto (u_1(x), u_2(x))$ (dashdot lines). In the inset we see that the performance of GMM is not reduced by the addition of the second informative eigenvector at the critical point $\eta(10)$.
}\label{fig:embedding_kmeans_bad}
\end{figure}

In Figure \ref{fig:embedding_kmeans_bad}, we performed some experiments regarding which clustering method to use on the spectral embedding. We simply used three popular methods, implemented in Python's Sklearn library (\cite{scikit-learn}): 
\begin{itemize}
\item k-means, the most popular method in graph clustering, 
\item Spectral-Clustering, which solves a norm-cut problem on the singular vectors of a distance matrix, a method known to be powerful when the clusters are non-convex, 
\item  Gaussian Mixture clustering, which fits the parameters of a mixture of gaussians to the data using the E-M algorithm. 
\end{itemize}

The first panel in Figure \ref{fig:embedding_kmeans_bad} is only a visual illustration of what spectral embeddings on a two-block SBM looks like. Here, the parameters are $\eta=0.99$ and $s=10$; our theory shows that there are two outliers in the spectrum of $A$. Our spectral embedding $\mathcal{X}$ in \eqref{spectral_embedding} has thus four dimensions (we use the left and right eigenvectors). For better visualization, we simply took the two right eigenvectors. Each point in the figure is thus $(u_1(x), u_2(x))$ for some node $x$ and the colors are the labels given by each clustering method.

The second panel in \ref{fig:embedding_kmeans_bad} shows the performance of these clustering methods, for $\eta$ between $1/2$ and $1$. We also compared the use of only one eigenvector with the use of two eigenvectors, even when there is only one informative eigenvector ($r_0=1$). 
\begin{enumerate}
\item When there is only one outlier ($\eta< \eta(s)$), clustering based on the Perron eigenvectors (solid lines) yields good results, while adding a second uninformative eigenvector (dashdot lines) deeply reduces the performance of any clustering method.
\item When crossing $\eta(s)$, a second informative eigenvector appears; the performance of clustering methods based on the Perron eigenvector are reduced, in accordance with the correlation decrease of $|\langle u_1, \varphi_1\rangle|$ predicted by Theorem \ref{thm:long} (see the golden line in Fig \ref{fig:MasterTheorem}, first panel).
\item But, when $\eta>\eta(s)$, the performance of kmeans and spectral-clustering based on the two informative eigenvectors $u_1, u_2$ first decreases, since these algorithms seem to struggle exploiting the extra information given by the second eigenvector (Figure \ref{fig:embedding_kmeans_bad}, top panel). Only the Gaussian mixture model incorporates this extra information efficiently: it is the only method for which clustering based on two informative vectors is \emph{better} than with only one (the two orange lines cross short after $\eta(s)$). 
\end{enumerate}

We did not try other clustering methods --- these experiments are only indicative of a seemingly high performance for gaussian clustering. In Theorem \ref{thm:eigenvector_convergence}, we showed that the spectral embeddings have a limiting distribution, accessible through the $Z_{i,j}$. If $Z_{i,j}$ was Gaussian, the performance of GMM would be completely understood, but as mentioned before Conjecture \ref{conjecture}, in the sparse regime, the limiting distributions $Z_{i,j}$ (or equivalently, the spectral embeddings) \emph{are not Gaussian}. The following remark explains why.

\begin{remark}\label{rk:nongaussien}
$Z_{i,j}$ has a positive atom at $0$ (and possibly many other atoms): indeed, following the notations of the very last subsection, it is easily seen that the limit $\mathscr{Z}_{i,j}$ is equal zero when the Galton-Watson tree $\mathscr{T}_j$ is empty, which happens with strictly positive probability so $\mathbf{P}(Z_{i,j}>0)>0$; but clearly, the extinction probability of $\mathscr{T}_j$ goes to zero when $\nu_1$, the highest modularity eigenvalue, goes to infinity.  Note that the atom at zero is visible in Figure \ref{fig:fluctuations}-(a). 
\end{remark}

\begin{figure}\centering
\begin{tabular}{c}
\includegraphics[width = 0.9\textwidth]{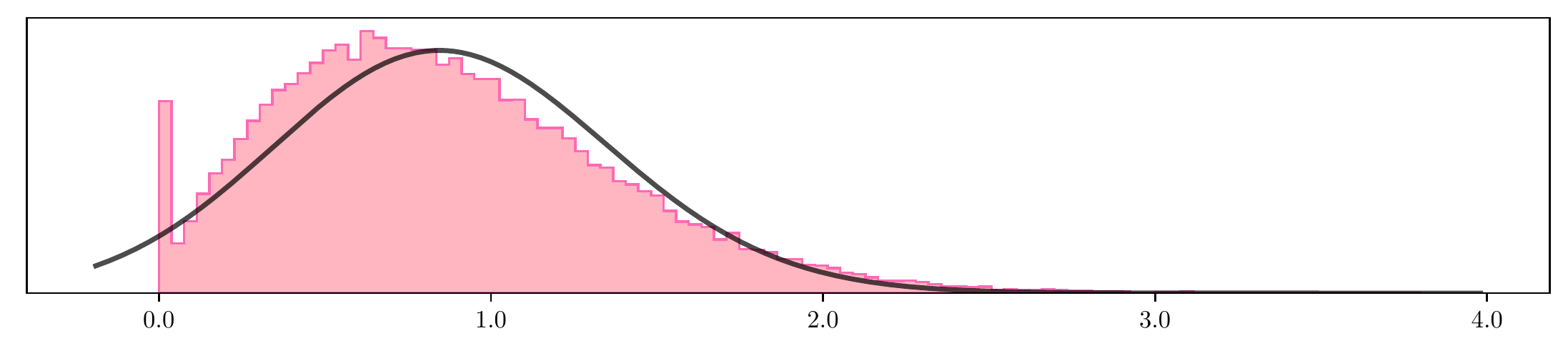}\\
(a): $r_0 = 1$. Histogram of the entries of $u_1$. \\
\includegraphics[width = 0.9\textwidth]{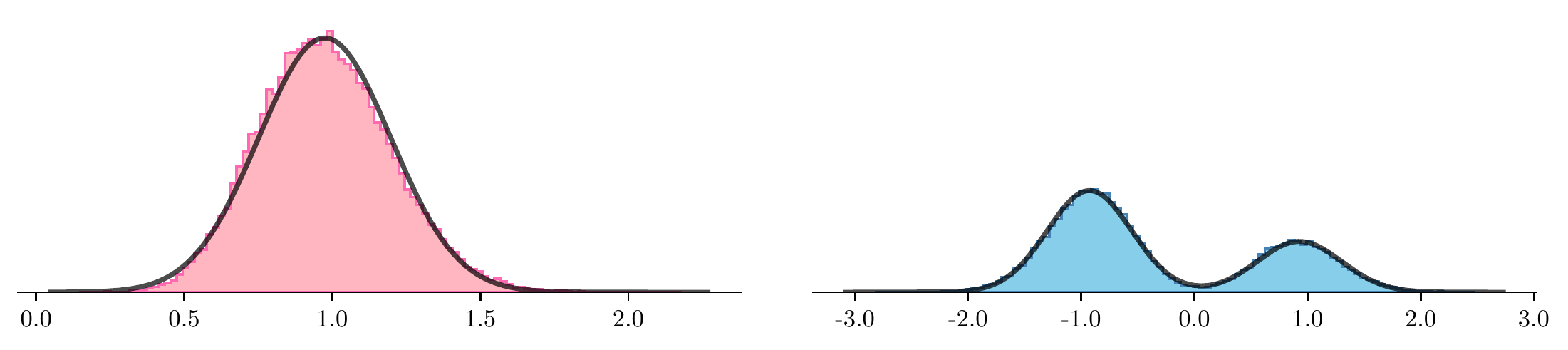}\\
(b) $r_0 = 2$. Histogram of $u_1$ in pink, and $u_2$ in blue. 
\end{tabular}
\caption{An illustration of Theorem \ref{thm:eigenvector_convergence}, for the two-block SBM with connectivities as in \eqref{fig:fluctuations} and clusters sizes  $p_1 = 2/3, p_2 = 1/3$. Here, $n=5000$ and there are 10 samples in each picture. The grey lines are the densities of the gaussian mixtures in \eqref{GM}.
}\label{fig:fluctuations}
\end{figure}

The shapes of the random variables $Z_{i,j}$ are visible in Figure \ref{fig:fluctuations}. In this figure, we plotted the histograms of the entries of $u_i$ in two different $5000$-nodes blockmodels, with connectivies $F_1$ (top) and $F_2$ (bottom) given by:
\begin{align}\label{F_fluc}
F_1 &= \begin{pmatrix}
6 & 4\\ 5 & 3
\end{pmatrix} &&F_2 = \begin{pmatrix}
48 & 6\\ 12 & 24
\end{pmatrix}
\end{align}
with two clusters of size $p_1 = 2/3, p_2 = 1/3$ --- the same size on the left and on the right. In the first model, there is only one outlier; in the second, there are two.

Note that the Theorem implies that for $i \in [r_0]$, the densities of the discrete distributions
\[\frac{1}{n}\sum_{x \in [n]}\delta_{u_i(x)} \]
converge in distribution to $p_1Z_{i,1}+p_2Z_{i,2}$.  The plots are the histograms of the $u_i(x)$ over 10 samples. In the figure, the grey lines are the densities of the \emph{gaussian} mixtures
\begin{equation}\label{GM}
\mathcal{N}_i = p_1 \mathscr{N}(\mu_{i,1}, \sigma_{i,1}^2)+p_2 \mathscr{N}(\mu_{i,2}, \sigma_{i,2}^2),
\end{equation}
where $i \in \{1, 2\}$ and $\mu_{i,j}, \sigma_{i,j}^2$ are the means and variances of the limiting random variables $Z_{i,j}$. It is clearly seen that in the first plot, the limit of $u_1$ is not Gaussian; in the second plot, the degrees of the graph are much higher (we are already on the semi-sparse regime), and the Gaussian approximations for both eigenvectors are strikingly convincing.

\newpage

\section{A bird's eye view on the proof of Theorem \ref{thm:main}}\label{app:proof_thm1}

The proof of Theorem \ref{thm:main} follows the celebrated high-trace method, introduced in \cite{laurent, BLM}. Considerable advances and simplifications have recently been made; on one hand, \cite{ludovic} performed this method on the \emph{non-backtracking} matrix of weighted, \emph{inhomogeneous undirected} graphs; on the other hand, \cite{simon} performed this method on the \emph{adjacency} matrix of weighted, \emph{homogeneous directed} graphs. In both of these papers, the underlying matrices $P$ and $W$ were Hermitian, which is no longer the case here. Our Master Theorem bridges the gap, and considers the \emph{adjacency} matrix of weighted, \emph{inhomogeneous directed} graphs with \emph{general} $P,W$. 

We hereby sketch the main ideas at a high level, and when needed we link our proof with the formerly cited papers. We emphasize that the proofs of theorems like Theorem \ref{thm:main} are often very technical. In this appendix, we tried to be as elementary as possible, to hide the technical details already written in other papers, and to give a short, accessible summary of the proof ideas --- at the cost of completeness.

\subsection{Warmup: notations}

\newcommand{\Gammal}{\Gamma_{\mathrm{left}}}

\newcommand{\Gammar}{\Gamma_{\mathrm{right}}}

\paragraph{Probabilistic domination. }For readability, we introduce notations regarding the asymptotic order of real random variables. Let 
\[ X=(X_n : n \in \mathbb{N}) \qquad \qquad Y=(Y_n : n \in \mathbb{N})\]
be two families of real random variables. We write $X \preceq Y$ if there is a constant $D$ such that for every constant $c>0$, for $n$ large enough, \[\mathbb{P}(|X_n|> \log(n)^D|Y_n|) \leqslant \log(n)^{-c}, \] in other words, $|X_n|$ is smaller than $|Y_n|$ up to logarithmic terms, with probability smaller than every polylogarithm (typically, $n^\delta$ for small $\delta$).  Finally, we write $X \ll Y$ if for every constants $c,D>0$, for $n$ large enough, 
\[ \mathbb{P}(|X_n| >\log(n)^{-D}|Y_n|) \leqslant \log(n)^{-c}.\]
In other words, $|X_n|/|Y_n|$ goes to zero faster than every polylogarithm. With this handy device, it is easily seen that $X_n \preceq Y_n$ and $Y_n \ll Z_n$ imply $X_n \ll Z_n$. These notations are common in the field of random matrix theory, and they truly simplify the exposition compared with \cite{ludovic, simon}.

\paragraph{Spectral decomposition. }Before starting, we write the spectral decomposition of $Q=\mathbf{E}[A]$:
\[Q = \sum_{i=1}^r \mu_i \varphi_i \psi_i^*. \]
The $\mu_i$ are the $r$ nonzero eigenvalues; we order them by decreasing modulus, $|\mu_1| > \dotsb > |\mu_r|$. The $\varphi_i$ are the corresponding unit right eigenvectors; they do not always an orthonormal basis, because $Q$ has not been supposed normal. The $\psi_i$ are the unique (up to a sign) left eigenvectors satisfying $\langle \varphi_i, \psi_j \rangle = \delta_{i,j}$, and in general they do not have unit-length. In the statement of the Master Theorem, the unit left eigenvectors $\xi_i$ are thus $\psi_i/|\psi_i|$.
 
\paragraph{Thresholds, spectral gap. }We recall that $r_0$ is the number of eigenvalues of $Q$ with modulus greater than $\thresh = \max (\Vert W \Vert_\infty, \sqrt{\rho})$ where $\rho=\sqrt{\Vert K \Vert}$ and we introduce $\Phi = (\varphi_1, \dotsc, \varphi_{r_0}), \Psi = (\psi_1, \dotsc, \psi_{r_0})$ and $\Sigma = \mathrm{diag}(\mu_1, \dotsc, \mu_{r_0})$. The spectral gap of our model is defined as
\[
  \tau = \sqrt{\frac{\thresh}{|\mu_{r_0}|}}.
\]
It is very important to note that $\tau<1$. The closer to $1$, the harder the problem; the bounds of Theorem \ref{thm:main} are actually of the form
$\left|\lambda_i - \mu_i \right| \preceq \tau^\ell$,
where $\ell$ is a carefully chosen parameter that grows logarithmically with $n$.

\paragraph{Covariance functionals, eigendefects and cross-defects} We will first need a notation for the Hadamard products of vectors in $\mathbb{R}^n$:
\begin{equation}\label{eq:phi_ii_def}
\varphi^{i,j} = \varphi_i \odot \varphi_j \qquad \text{ and } \qquad \psi^{i,j} = \psi_i \odot \psi_j, 
\end{equation}
so that in the statement of the Master Theorem, $\varphi_i^2$ is equal to $\varphi^{i,i}$. In the proof we will only use the $\varphi^{i,i}$ notation.

 We introduce two functions $\Gamma, \Gamma^* : \mathbb{C} \times \mathbb{C}^n \to \mathbb{C}$ by 
\begin{equation}\label{gamma}
  \Gamma(z,\xi) = \sum_{t = 0}^\infty \frac{\langle \mathbf{1}, K^t \xi \rangle}{z^t} \qquad \text{ and } \qquad
  \Gamma^*(z,\xi) = \sum_{t = 0}^\infty \frac{\langle \mathbf{1}, (K^*)^t \xi \rangle}{z^t}.
\end{equation}
Let us recall that the norm of $K$ is $\rho$; consequently, the sums above are convergent when $|z|$ is strictly greater than $\rho$, and in particular when $z=\mu_i^2$ with $i \in [r_0]$. In this case, one has
\begin{equation}
\label{link_eigendefect_gamma}
\Gamma(\mu_i^2,\varphi^{i,i}) =  \mu_i^2 \sqrt{R_i} \qquad \text{ and }\qquad \Gamma^*(\mu_i^2,\psi^{i,i}) =  \mu_i^2 \sqrt{L_i}. 
\end{equation}
\begin{proof}Let $\xi = \varphi^{i,i}$. We have 
$\Gamma(\mu_i^2,\xi) =\sum_{t = 0}^\infty \frac{\langle \mathbf{1}, K^t \xi \rangle}{\mu_i^{2t}} = \langle \mathbf{1}, (\sum_{t = 0}^\infty  (K/\mu_i^2)^t) \xi \rangle$. The Neumann summation formula shows that since $\Vert K/\mu_i^2 \Vert <1$, the matrix sum in this expression is nothing but $(I - K/\mu_i^2)^{-1} = \mu_i^2(\mu_i^2 I - K)^{-1}$, and that this matrix has all entries nonnegative. Since $\xi$ also has nonnegative entries, $|(\mu_i^2 I - K)^{-1}\xi|_1 = \langle \mathbf{1}, (\mu_i^2 I - K)^{-1}\xi \rangle$, and we recognize the right eigendefects. 
\end{proof}
Since we'll also need `cross defects' like $\Gamma(\mu_i \mu_j, \varphi^{i,j})$, we introduce the notations $\Gammal, \Gammar$ for two matrices of size $r_0\times r_0$ defined by
\[
  (\Gammal)_{i,j} = \Gamma(\mu_i \mu_j, \varphi^{i,j}) \qquad \text{ and } \qquad (\Gammar)_{i,j} = \Gamma^*(\mu_i \mu_j, \psi^{i,j}).
\]

\subsection{The Pseudo-Master Theorem: \texorpdfstring{$A$}{A} is nearly diagonalized by pseudo-eigenvectors}

The tools for studying matrix $A$ are the \emph{pseudo-eigenvectors}. Consider the eigenvalue $\mu_i$ of $Q$. We define two vectors $U_i = A^\ell \varphi_i/\mu_i^\ell$ and $V_i = (A^*)^\ell \psi_i/\mu_i^\ell$, where 
\begin{equation}
\ell = \lfloor \kappa \log(n) \rfloor 
\end{equation}
and $\kappa$ is a positive constant to be chosen later. To put it in matrix form, these vectors are the columns of the $n \times r_0$ matrices
\begin{equation*}
  U = A^\ell \Phi \Sigma^{-\ell} \qquad \text{ and } \qquad V = (A^*)^\ell \Psi \Sigma^{-\ell}.
\end{equation*}

The key aspects of $U$ and $V$ are summarized in the following list of statements, which is a pseudo-version of the Master Theorem. We recall that the symbols $\ll, \preceq$ where rigorously defined in the preceding subsection. 
\begin{theorem}[Pseudo-Master Theorem]For a sufficiently small choice of $\kappa$, the following holds.
\begin{enumerate}
  \item $U$ and $V$ are nearly inverses:
        \begin{equation}\label{UV}
          \Vert U^* V - I \Vert \ll 1
        \end{equation}
  \item $U$ and $\Psi$ are nearly inverses, and $V$ and $\Phi$ are nearly inverses:
        \begin{equation}\label{PHIU}
          \Vert \Psi^* U - I \Vert \ll 1 \qquad \qquad \Vert \Phi^* V - I\Vert \ll 1
        \end{equation}
  \item $V$ and $U$ nearly diagonalize $A^\ell$:
        \begin{equation}\label{VAU}
          \Vert V^* A^\ell U - \Sigma^\ell \Vert\ll 1
        \end{equation}
          \item $\Gammal$ is nearly the Gram matrix of $U$, and the same for $\Gammar$ and $V$:
        \begin{equation}\label{UU}
          \Vert U^*U - \Gammal \Vert \ll 1 \qquad  \qquad \Vert V^*V - \Gammar \Vert \ll 1
        \end{equation}
  \item $A^\ell$ is negligible outside of the vector spaces spanned by the pseudo-eigenvectors:
        \begin{equation}\label{dur}
          \Vert A^\ell \mathrm{Proj}_{\mathrm{im}(V)^\perp} \Vert \preceq \thresh^\ell, \qquad \qquad \Vert  \mathrm{Proj}_{\mathrm{im}(U)^\perp} A^\ell\Vert \preceq \thresh^\ell,
        \end{equation}
        where $\mathrm{Proj}_{C^\perp}$ denotes the projection matrix on the orthocomplement of the subspace $C$.
\end{enumerate}
\end{theorem}

A crucial point in this theorem is that the error scale (up to polylog factors) is $\thresh$: only the last bound, \eqref{dur}, is actually sharp. The other error terms are meant to be negligible.

Proving the Pseudo-Master Theorem (PMT) is really the core of the proof, and where lie most of the difficulties. From the PMT, it is only a matter of linear algebra and perturbation theory to prove the Master Theorem: we simply summarize the spirit in the next subsection, and we quickly jump to the proof of the PMT. 

\subsection{Master Theorem = Pseudo-Master Theorem + perturbation theory}
Given the statements in the preceding subsection, the main theorem follows from a standard algebraic analysis, which builds on detailed quantitative variants of the Bauer-Fike theorem. We refer the reader to the comprehensive studies in \cite[Section 4]{ludovic} or \cite[Section 8]{simon}, which are technical, but can be applied directly without any further modification. In this short paragraph, we simply explain the ideas.

The main trick is to define a matrix $S = U\Sigma^\ell V^*$. If we had $V^* U = I$, this matrix would be exactly diagonalizable with eigenvalues $\mu_i^\ell$; but $V^* U$ is only close to $I$. Fortunately, it is easily seen that if $U,V$ are well-conditioned, then $S$ is diagonalizable, with eigenvalues close to the $\mu_i^\ell$ and eigenvectors close to the $U_i, V_j$. The fact that $U,V$ are well-conditioned follows from \eqref{UU}: by continuity, their condition number is close to the condition number of $\Gammar, \Gammal$, who in turn are bounded: 
\begin{lemma}
The condition numbers of the matrices $\Gammar, \Gammal, U, V$ are all $\preceq 1$.
\end{lemma}

\begin{proof}
Note $\pi_s(x) = \sqrt{(K^*)^s \mathbf{1}(x)}$ and $\Pi_s = \mathrm{diag}(\pi_s)$. It is easily seen that $\langle \mathbf{1}, K^s \varphi^{i,j} \rangle = (\Phi^* \Pi^2_s \Phi)_{i,j}$. The matrix $\Gammal$ is thus a sum of semi-definite positive matrices, with first term $I_0$, so its smallest eigenvalue is $\geqslant 1$ and its condition number is smaller than $\Vert \Gammal \Vert$. On the other hand, note that $|\langle \mathbf{1}, K^s \varphi^{i,j}\rangle | \leqslant |\mathbf{1}||\varphi^{i,j}| \rho^s$. Thanks to \eqref{hyp:deloc}, we have $|\varphi^{i,j}| \leqslant c\sqrt{1/n}$ for some $c$; consequently, $(\Gammal)_{i,j} \leqslant c \sum \rho^s / (\mu_i\mu_j)^s \leqslant c/(1-\tau) \preceq 1$. Finally, since the size $r_0$ is also $\preceq 1$, we have $\Vert \Gammal \Vert \preceq 1$ and the condition number is $\preceq 1$.  By the Weyl inequalities and \eqref{UU}, $|\mathrm{cond}(U)-\mathrm{cond}(\Gammal)| \leqslant \Vert UU^* - \Gammal\Vert \ll 1$ and $\mathrm{cond}(U) \preceq 1$.
\end{proof}

The smallest eigenvalue $\nu_{r_0}$ of $S$ is close to $\mu_{r_0}^\ell$, and since $\tau>1$ independently of $n$ and $\ell =\lfloor \kappa \log(n) \rfloor$, we get $|\mu_{r_0}^\ell|\ll \thresh^\ell $. But we can write $A^\ell$ as a perturbation of $S$,
\[
  A^\ell = S + (A^\ell - S),
\]
and, crucially, statements \eqref{VAU}-\eqref{dur} can be bootstraped to show that $\Vert A^\ell - S \Vert \ll 1$. The Bauer-Fike theorem applies, and roughly says that the eigenvalues of $A^\ell$ are within distance $\ll \thresh^\ell$ of the eigenvalues of $S$: consequently, $A^\ell$ has $r_0$ eigenvalues $\thresh^\ell$-close to $\mu^\ell_1, \dotsc, \mu_{r_0}^\ell$, and the eigenvalue $0$ of $S$ gives rise to $n- r_0$ eigenvalues of $A^\ell$ with modulus $\preceq \thresh^\ell$. A similar statement holds for the eigenvectors; extra work needs to be done for getting results on the eigenvalues of $A$, since there might be some phase effects.

The results on eigenvectors follow in the same way, with a Davis-Kahan-like custom theorem proved in \cite{ludovic}. To ensure performant bounds, we must assume that the eigenvalues $\mu_i$ are well-separated, which the reason why we suppose that $|\mu_i - \mu_j|> c$ for some $c>0$ in the Hypotheses before the Master Theorem. Since $\Vert A^\ell - S \Vert \ll 1$ and the condition numbers of the matrices of interest are all $\preceq 1$, the Davis-Kahan theorem yields a bound of the form $|u_i - U_i/|U_i|| \preceq \Vert A^\ell - S \Vert / c$, which is $\ll 1$:
\begin{equation}
\label{app:eigvec_cv}
\left| u_i - \frac{U_i}{|U_i|}\right| \ll 1.
\end{equation}
The `eigenvector part' of the Master Theorem then easily follows, by the continuity of $u \mapsto \langle u, \varphi\rangle$ and the limits in the Pseudo-Master Theorem.

\section{Proof of the Pseudo-Master Theorem}\label{app:PMT}

This section is devoted to the proof of the Pseudo-Master Theorem. 

The first two sections gather some results on local approximations of random graphs: to a large extent, they are classical and well-known. We state them for completeness and because they give a good intuition on the following parts, but these statements can be retrieved using routine methods from random graph theory. Subsection \ref{subsec:concentration} states a powerful concentration result on random graphs, proved in \cite{simon}. 

We use these results as a computational toolbox: in Subsections \ref{sec:poisson} to \ref{sec:correlation}, we perform all the necessary calculations on our pseudo-eigenvectors, and we rely on an elegant generalization of Kesten martingales on inhomogeneous Galton-Watson trees. These sections differ from previous works in the sense that we had to adapt the computations to the general setting of our Master Theorem, with full inhomogeneity and non-symmetry.

 Finally, the main ideas of previous works on trace methods are summarized in Subsection \ref{subsec:trace}, as they apply without modifications to our setting.

\subsection{The graph has few short cycles, and small neighbourhood growth}Sparse random graphs, that is, random graphs where the mean degree of vertices is $O(1)$, have been known for long to be locally-tree like, in the sense they have very few short cycles. In our model, we supposed that $P_{x,y} \leqslant d/n$ for some $d$. As a consequence, the expected degree of a vertex $x$ is $d_x = P_{x, 1}+ \dotsb + P_{x, n}$, and is smaller than $d$, and the graph is stochastically dominated by a directed homogeneous Erd\H{o}s-R\'enyi random graph with connectivity $d/n$. In turn, many random variables (cycle counts, edge number) are stochastically dominated by the undirected Erd\H{o}s-Rényi graph of degree $2d/n$. Most properties on the local structure directly follow from classical results: absolutely no problem-specific work is needed here. We simply gather the results we will tacitly use in the sequel.

We note $(G, x)_t$ the \emph{forward} neighbourhood of radius $t$ around $x$ in $G$, that is: the subgraph spanned by vertices $y$, for which there is a directed path with length smaller than $t$ from $x$ to $y$. The crucial choice will be the depth $\ell$ at which we look into the graph. We recall
\begin{equation*}\label{def:ell}
  \ell = \lfloor \kappa \log(n) \rfloor
\end{equation*}
where $\kappa$ is an explicit constant depending on $d$. With such a choice, we have the crucial property that shallow neighborhoods are nearly trees. We say that a graph is $t$-\emph{tangle-free} if for every vertex $x$, the subgraph $(G, x)_{t}$ has no more than one directed cycle.
\begin{lemma}\label{tanglefree}
  There is a constant $c=c(d)$ such that $G$ is $2\ell$-tangle-free with probability $ \geqslant 1 - n^{-c}$. Moreover, if $N$ is the number of vertices $x$ such that $(G, x)_{2\ell}$ contains a cycle, then $N \preceq 1$.
\end{lemma}

The proof Lemma \ref{tanglefree} follow from the choice of the constant $\kappa$ in \eqref{def:ell}. We refer to \citep{BLM} for the details.

\subsection{The graph is locally approximated by trees}\label{subsec:coupling} Let $x$ be a vertex in $G$, and $\mathcal N(x)$ the set of its neighbours (in the forward sense: $y$ is a neighbor or $x$ if $(x,y) \in E$, not necessarily when $(y,x) \in E$). Then $\mathcal N(x)$ has the following distribution : each vertex $y \neq x$ is present in $\mathcal N(x)$ with probability $P_{x,y}$, independently from all other vertices. It is a well-known fact that whenever the $P_{x,y}$ are small, the law of $|\mathcal N(x)|$ is well approximated by $\mathrm{Poi}(d_x)$ (the so-called `rare events theorem'). Moreover, conditionnally on $\{\mathcal{N}(x) \neq \varnothing \}$, the distribution of a random element in $\mathcal{N}(x)$ is equal to $\pi_x$, where 
 \begin{equation}\label{eq:def_pix}
    (\pi_x)_z = P_{x,z}/d_x.
 \end{equation}
 That being said, the distribution of $k$ elements in $\mathcal{N}(x)$ is \emph{not} the product distribution $(\pi_x)^{\otimes k}$ because the elements have dependencies (they cannot be chosen multiple times, for instance), but these dependencies are nearly nonexistent. In fact, let us introduce a new distribution $\mathbb{Q}_x$ on random multi-sets.
\begin{itemize}
  \item The number of elements of the random multiset $\mathcal{M}(x)$ under $\mathbb{Q}_x$ is a Poisson random variable with mean $d_x$;
  \item Conditionnally on $|\mathcal{M}(x)|=k$, each of these $k$ element is sampled independently on $[n]$ with probability distribution $\pi_x$.
\end{itemize}

The following proposition is a rigorous formulation of the intuitions given above. We set $\sigma_x = P^2_{x,1}+ \dotsb + P^2_{x,n}$; in our regime, $\sigma_x = O(1/n) = o(1)$.

\begin{proposition}{\cite[Lemma 8]{ludovic}}\label{prop:dtvMN}
  We have
  \begin{equation}
    d_{\mathrm{TV}}(\mathbb{P}_x, \mathbb{Q}_x)  \leqslant 2\sigma_x.
  \end{equation}
\end{proposition}

Building from those remarks, we define a random tree $T_x$ as follows :
\begin{itemize}
  \item the root (at depth $0$) is a single vertex labeled $x$,
  \item for each vertex $i$ at depth $t$ with label $x_i$, the children of $i$ at depth $t+1$ with their labels have the joint distribution $\mathcal M(x_i)$, independently from all other vertices at depth $\leqslant t$.
\end{itemize}
With this definition, the tree $T_x$ is formally undirected, but we can view it as a directed tree with directions flowing out of the root. Proposition \ref{prop:dtvMN} then implies that the neighbourhood distributions in $T_x$ and $G$ are similar, as summarized in the following result:

\begin{lemma}\label{prop:tree_coupling}
  Whenever $\ell = \lfloor \kappa \log(n) \rfloor$ with $\kappa$ small enough, we have
  \begin{equation}\label{app:eq:tree_coupling} d_{\mathrm{TV}}((G, x)_{2\ell}, (T_x, x)_{2\ell}) \ll 1. \end{equation}
\end{lemma}

This lemma will not directly be used in our proof; it is only a step in the proof of Proposition \ref{lem:coupling_functionals} thereafter. We stated it anyways because it gives a rigorous meaning to the fact that $G$ is well-approximated locally by trees, and it intuitively gives a justification for our tree computations in subsequent parts of the proofs. 

Note that we formulated this section only with forward neighborhoods; the propositions are also true with backward neighborhoods, and in this case we only have to see $T_x$ as a directed tree, with edges oriented towards the root.

\subsection{Concentration of linear functionals}\label{subsec:concentration}

The graph $G$ (or the matrix $A$) is a collection of $n^2$ independent random variables: it is thus natural that if $F(G,x)$ is a function which depends only on a small neighborhood of the graph, then its space average 
\begin{equation}
\label{app:spatial_avg}
\frac{1}{n}\sum_{x \in [n]} F(G,x) 
\end{equation}
should be concentrated around its mean. This is the case, and a stronger statement actually holds: we saw in \eqref{app:eq:tree_coupling} that $(G,x)$ and $(T_x,x)$ nearly have the same distribution, and it is actually known that the functional in \eqref{app:spatial_avg} is indeed concentrated around the expectation of the same functional applied on the trees $(T_x, x)$. 

To formalize this, we say that a function $f : \mathscr G \times \mathbb N \mapsto \mathbb R$, where $\mathscr G$ denotes the set of all graphs, is $\ell$-local if $f(H, x)$ only depends on the $\ell$-neighbourhood of $x$ in $H$. Let $(T_x, x)$ be the family of independent random rooted trees as in \eqref{app:eq:tree_coupling}.
Then the following lemma is true as long as the constant $\kappa$ in $\ell$ is small enough ($\kappa<0.01/\log(2d)$ will be sufficient).

\begin{proposition}\label{lem:coupling_functionals}
  Let $\mathscr{F}$ be a family of functions $f : \mathscr G \times \mathbb N \mapsto \mathbb R$, with less than $\preceq 1$ elements.  We suppose that each $f \in \mathscr{F}$ is a $2\ell$-local function, that for all graphs $H$ and node $x$ one has  $\sup_{f \in \mathscr{F}}f(H, x) \leqslant  |(H, x)_{2\ell}|^2 \times (c_n/n) $, for some $c_n>0$. Then

  \[ \sup_{f \in \mathscr{F}}\left|\sum_{x \in G} f(G, x) - \sum_{x \in [n]} \mathbf E\left[f(T_x, x)\right] \right| \preceq \frac{c_n}{n^{0.1}}. \]
\end{proposition}

\begin{proof}
This is a restatement of Theorem 12.5 in \cite{simon}, with $\beta=2$ and $\alpha=c_n/n$. The bound therein is $c_n n^{-1/2 + 2\kappa}$. A sufficiently small $\kappa$ leads to our statement. Note that if $f(H, x) \leqslant  |(H, x)_{2\ell}|$, then $f$ also satisfies $f(H, x) \leqslant  |(H, x)_{2\ell}|^2$, which is the reason why we chose to keep the exponent $2$. \end{proof}

A crucial point in the proof of the Pseudo-Master Theorem is that the {pseudo-eigenvectors} $U_i$ and $V_i$ are $\ell$-local functions of $G$, and thus so are their scalar products with one another. As such, all computations in \eqref{UV}-\eqref{VAU} reduce to computing expectations on the random trees defined in the preceding subsection. We can apply Lemma \ref{lem:coupling_functionals} to $G$ and the family $(T_x)_x$; the problem was reduced to computing the expectations of our functionals on the trees $T_x$.

\subsection{Pseudo-eigenvectors on the random tree}\label{sec:poisson}
Let us take a quick look at the entries of $U$ and $V$, and pick one vertex $x \in [n]$ and some index $i$. By definition,
\begin{align}
U_i(x) & = \mu_i^{-\ell}(A^\ell \varphi_i)(x) \nonumber \\
& = \mu_i^{-\ell} \sum_{x_1, \dotsc, x_\ell}W_{x, x_1}W_{x_1, x_2} \dotsb W_{x_{\ell-1}, x_\ell} \varphi_i(x_\ell) \label{eq:ui},
\end{align}
where the sum runs over all the paths in the graph $G$, ie sequences of vertices such that the edges $(x_s, x_{s+1})$ are present in the graph.

This is, by definition, an $\ell$-local function ; its counterpart on the tree can thus be defined as
\begin{equation}\label{eq:tilde_ui}
  \tilde U_i(x) = \mu_i^{-\ell} \sum_{x_1, \dots, x_\ell \in T_x} W_{x, \iota(x_1)} W_{\iota(x_1), \iota(x_2)} \dotsb W_{\iota(x_{\ell-1}), \iota(x_\ell)} \varphi_i(\iota(x_\ell)),
\end{equation}
where the sum ranges over the vertices $x_\ell$ at depth $\ell $ in $T_x$ and the unique path $x, \dots, x_\ell$ connecting $x$ to $x_\ell$, and $\iota(x_i)$ is the label of $x_i$.

\subsection{The martingale equation}\label{sec:martingale_equation}

Let $\mathbf E_{\ell-1}$ denote the conditional expectation with respect to the first $\ell - 1$ generations of $T_x$. Write
\begin{equation*}
  \mathbf{E}_{\ell - 1}[\tilde U_i(x)] = \mu_i^{-\ell}\sum_{x_1, \dotsc, x_{\ell-1}}W_{x, \iota(x_1)} \dotsb W_{\iota(x_{\ell-2}), \iota(x_{\ell-1})} \mathbf{E}_{\ell-1}\left[\sum_{x_\ell}W_{\iota(x_{\ell-1}), \iota(x_\ell)} \varphi_i(\iota(x_\ell))\right]
.
\end{equation*}
Let us note $y = \iota(x_{\ell-1})$; the inner expectation reads
\[ \mathbf{E}_{\ell-1}\left[\sum_{x_\ell}W_{\iota(x_{\ell-1}), \iota(x_\ell)} \varphi_i(\iota(x_\ell))\right] = \mathbf E\left[\sum_{z \in \mathcal M(y)} W_{y,z}\varphi_i(z)\right] \]
This is a sum of a $\mathrm{Poi}(d_y)$ number of independent random variables with distribution $\pi_y$, so
\begin{align*}
  \mathbf{E}\left[  \sum_{z \in \mathcal{M}(y)} W_{y,z}\varphi_i(z) \right] & = d_y \times
  \sum_{z \in [n]} \frac{P_{y,z}}{d_y}W_{y,z} \varphi_i(z)                                                         \\
  & = \sum_{y \in [n]} P_{y,z}W_{y,z} \varphi_i(z) \\
  & = \sum_{y \in [n]}Q_{y,z}\varphi_i(z) \\
  & = (Q\varphi_i)(y)                      \\
  & = \mu_i \varphi_i (y)
\end{align*}
where, in the last line, we used the fact the $\varphi_i$ is an eigenvector of $Q$. Recalling that $y = \iota(x_{\ell - 1})$, we have
\[ \mathbf{E}_{\ell - 1}[\tilde U_i(x)] = \mu_i^{-\ell+1}\sum_{x_1, \dotsc, x_{\ell-1}}W_{x, \iota(x_1)} \dotsb W_{\iota(x_{\ell-2}), \iota(x_{\ell-1})} \varphi_i(\iota(x_{\ell-1}))
. \]

In fact, defining $\tilde U_i(x, t)$ by replacing $\ell$ by $t$ in the definition of $\tilde U_i(x)$, we showed that
\begin{equation}\label{martingale}
  \text{The random process } t \mapsto \tilde U_i(x, t) \text{ is a martingale}.
\end{equation}
The common expectation is easily seen to be $\mathbf E[\tilde U_i(x, 0)] = \varphi_i(x)$.

\subsection{Proving \texorpdfstring{\eqref{UV}-\eqref{PHIU}-\eqref{VAU}}{the covariance part of the PMT}} We only give the argument for \eqref{UV}, the others are done in a similar fashion. The $(i, j)$ coefficient of $U^*V$ is $\langle U_i, V_j \rangle$, which can be rewritten as
\[ \langle U_i, V_j \rangle = (\mu_i \mu_j)^{-\ell} \langle A^\ell \varphi_i, (A^*)^\ell \psi_j \rangle = (\mu_i \mu_j)^{-\ell} \langle A^{2\ell} \varphi_i, \psi_j \rangle \]
Let $f(x) = f_{i,j}(x) = \psi_j(x)[A^{2\ell}\varphi_i](x)$; it is easily seen that $\langle U_i, V_i\rangle = \sum_{x \in [n]}f(G,x)$, and we are ready to apply the concentration property in Proposition \ref{lem:coupling_functionals} to the family $\{f_{i,j}\}$.
\begin{lemma}[correlations between pseudo-eigenvectors are concentrated]
\begin{equation}
\sup_{i,j} \left| \langle U_i, V_j \rangle - \mu_i^{-\ell} \mu_j^{-\ell} \sum_{x \in [n]} \psi_j(x) \mathbf E[\tilde U_i(x, 2\ell)] \right| \ll 1
\end{equation}
\end{lemma}

\begin{proof} The delocalization properties of the $\varphi_i$ (Hypothesis \eqref{hyp:deloc}), the tangle-free property (there is no more than one cycle in $(G,x)$) and the fact that $|W_{x,y}|\leqslant \Vert W \Vert_\infty $ all together imply that
  $$f(G, x) \leqslant  \frac{c\Vert W \Vert^\ell_\infty}{ n} |(G, x)_{2 \ell}|,$$
  for some universal constant $c$. But since $\Vert W \Vert_\infty < c'$ for some $c'$, any choice for $\kappa$ sufficiently small will give (for instance) $\Vert W \Vert_\infty^\ell \preceq n^{0.01}$. Proposition \ref{lem:coupling_functionals} straightforwardly leads to a $\preceq n^{-\varepsilon}$ error for some small $\varepsilon$. With our notations, this is $\ll 1$.

\end{proof}
We are now in a position to use property \eqref{martingale}: $\mathbf E[\tilde U_i(x, 2\ell)] = \varphi_i(x)$, and the orthogonality property of the left and right eigenvectors yields $\left|\langle U_i, V_j \rangle - \delta_{ij} \right| \ll 1$. It is then straightforward to go from this elementwise bound to \eqref{UV}: for any $r_0 \times r_0$ matrix $M$, one has $\Vert M \Vert \leqslant r \Vert M \Vert_\infty$. Since we also have $r \preceq 1$, we obtain $\Vert U^* V - I \Vert \preceq \Vert U^* V - I \Vert_\infty \ll 1$.

\bigskip

The same proof works for $V$, with the subtle difference that these left-pseudo-eigenvectors $V_i$ are backward-looking: $V_i(x)$ is a local function of $(G,x)_\ell^-$, where the $-$ superscript denotes the backward ball. But the proof is the same: Lemma \ref{lem:coupling_functionals} is true for backward functionals and the couplings in Subsection \ref{subsec:coupling} are the same, but with the $T_x$ oriented towards the root; everything works exactly the same. 

\subsection{Proving \texorpdfstring{\eqref{UU}}{the auto-covariance part of PMT}: martingale correlations}\label{sec:correlation}

Statements in \eqref{UU} are trickier: even if Lemma \ref{lem:coupling_functionals} still allows approximating $\langle U_i, U_j \rangle$ with a tree computation, there are some strong dependencies between $\tilde U_i(x, t)$ and $\tilde U_j(x, t)$ that we cannot neglect.

\paragraph{Rewriting the correlation term.} Proceeding as before with $f(G, x) = [A^\ell \varphi_i](x) \cdot [A^\ell \varphi_j](x)$, we have
\begin{equation}\label{eq:UU*_proof}\sup_{i,j} \left| \langle U_i, U_j \rangle - \sum_{x\in [n]} \mathbf E \left[ \tilde U_i(x) \tilde U_j(x) \right] \right| \ll 1. \end{equation}
We recognize a covariance term between two martingales; let us then introduce the increments
\[ \Delta_t = \mathbf{E}_{t-1}\left[(\tilde U_i(x, t) - \tilde U_i(x, t-1))(\tilde U_j(x, t) - \tilde U_j(x, t-1))\right] \]
A classical use of the martingale property implies that
\[ \mathbf E \left[ \tilde U_i(x) \tilde U_j(x) \right] = \varphi_i(x)\varphi_j(x) + \mathbf E[\Delta_1 + \dots + \Delta_\ell] \]

The increment $\Delta_t$ has an explicit expression:
\begin{multline}\label{increments}
  \Delta_t = \mu_i^{-t}\mu_j^{-t}\sum_{\substack{x_1, \dotsc, x_{t-1} \\ x'_1, \dotsc, x'_{t-1}}}\prod_{s=1}^{t-1}W_{\iota(x_{s-1}), \iota(x_{s})}W_{\iota(x'_{s-1}), \iota(x'_s)}\times \\ 
  \mathbf{E}_{t-1}\left[\sum_{\substack{x_{t-1}\to x_t \\ x'_{t-1} \to x'_t}}W_{\iota(x_{t-1}), \iota(x_{t})}W_{\iota(x'_{t-1}), \iota(x'_t)}\varphi_i(\iota(x_{t}))\varphi_j(\iota(x'_{t})) - \mu_i \mu_j\varphi_i(\iota(x_{t-1}))\varphi_j(\iota(x'_{t-1}))\right].
\end{multline}
\paragraph{Covariance of Poisson sums.}  In the sum above, the only nonzero terms are when $x_{t-1} = x_{t-1}'$: otherwise, the inner sum becomes a product of two independent random variables of respective expectations $\mu_i \varphi_i(x_{t-1})$ and $\mu_j \varphi_j(x_{t-1}')$. Writing again $y = \iota(x_{t-1})$, the conditional expectation becomes
\begin{equation}\label{eq:covariance_conditional}
  \begin{split} 
    C_y :&= \mathbf E\left[\sum_{z, z' \in \mathcal M(y)} W_{y,z} W_{y,z'} \varphi_i(z) \varphi_j(z') - \mu_i \mu_j \varphi_i(y)\varphi_j(y)\right] \\
    &= \mathrm{Cov}\left( \sum_{z \in \mathcal{M}(y)} W_{y,z}\varphi_i(z),  \sum_{z \in \mathcal{M}(y)} W_{y,z}\varphi_j(z) \right).
  \end{split}
\end{equation}

We then make use of the following elementary lemma:
\begin{lemma}\label{lemma:covariance_poisson}
  If $N$ is a Poisson random variable and $(A_1, B_1), (A_2, B_2), \dotsc$ are iid copies of a couple of random variable $(A,B)$, then
  \begin{equation}
    \mathrm{Cov}\left(\sum_{k=1}^N A_k, \sum_{k=1}^N B_k \right) = \mathbf{E}[N] \mathbf{E}[AB].
  \end{equation}
\end{lemma}
In \eqref{eq:covariance_conditional}, $N = |\mathcal M(y)|$ is a $\mathrm{Poi}(d_y)$ random variable, and the couple $(A, B)$ is simply
\[ A = W_{y, Z} \varphi_i(Z) \qquad \text{and} \qquad B = W_{y, Z} \varphi_j(Z), \]
where $Z$ is a random index on $[n]$ with distribution $\pi_y$. Computing $C_y$ is now straightforward:
\begin{align*}
  C_y &= d_y \mathbf E[W_{y, Z} \varphi_i(Z)W_{y, Z} \varphi_j(Z)] \\
      &= d_y \sum_{z \in [n]}\frac{P_{yz}}{d_y} W_{yz}^2 \varphi_i(z)\varphi_j(z) \\
      &= (K\varphi^{i, j})(y).
\end{align*}

Let us now come back to \eqref{increments}. We know that we can remove every term with $x_{t-1} \neq x_{t-1}'$, and the computations above further reduce it to
\[ \Delta_t = (\mu_i\mu_j)^{-t}\sum_{x_1, \dots, x_{t-1}} \prod_{s=1}^{t-1} W_{\iota(x_{s-1}, \iota(x_s))}^2 (K\varphi^{i, j})(\iota(x_{t-1})) \]
We recognize an expression similar to the definition of $\tilde U_i$ in \eqref{eq:tilde_ui}. Using the same methods, we are able to show that
\[ \mathbf E[\Delta_t] = \frac{[K^t\varphi^{i, j}(x)]}{(\mu_i\mu_j)^t}. \]
Finally, summing the increments, we get
\begin{align}
  \mathbf{E}[\tilde U_i(x)\tilde U_j(x)] & = \varphi_i(x)\varphi_j(x) + \sum_{t=1}^\ell K^t \varphi^{i,j}(x) \nonumber \\
                               & = \sum_{t=0}^\ell \frac{K^t \varphi^{i,j}(x)}{\mu_i^t \mu_j^t}. \label{eq:tilde_uij}
\end{align}
The expression for $\langle U_i, U_j \rangle$ is obtained by summing \eqref{eq:tilde_uij} over all values of $x$ and plugging this into \eqref{UU}, thus obtaining:
\begin{equation}\label{eq:uij_truncated}
\sup_{i,j} \left|  \langle U_i, U_j \rangle - \sum_{t=0}^\ell \frac{\langle \mathbf 1, K^t \varphi^{i,j} \rangle}{(\mu_i \mu_j)^t}  \right| \ll \thresh^\ell.
\end{equation}
Statement \eqref{eq:uij_truncated} is pretty close to \eqref{UU}, at the sole difference of the summation index, which is stopped at $\ell$. Also, note that \eqref{eq:uij_truncated} is valid for every $i,j$ in $[r]$, not just in $[r_0]$. However, whenever $i, j \in [r_0]$, we have by definition
\[ \sqrt{\rho} \leqslant\sqrt\tau \mu_i \quad \text{and} \quad \sqrt{\rho} \leqslant\sqrt\tau \mu_j, \]
and the spectral radius of the matrix $K/(\mu_i\mu_j)$ is thus at most $\tau$. Since the entries of $\varphi^{i, j}$ are of order $O(n^{-1})$, we have the bound
\[ \sup_{i,j}\left| \sum_{t = \ell + 1}^\infty \frac{\langle \mathbf 1, K^t \varphi^{i,j} \rangle}{(\mu_i \mu_j)^t} \right| \leqslant C \tau^{\ell} \ll 1. \]
Combined with \eqref{eq:uij_truncated}, this finally ends the proof of \eqref{UU}.

\subsection{The trace method}\label{subsec:trace}

All that remains now is to prove equation \eqref{dur}; that is, once we showed that the first $r_0$ eigenvalues of $A^\ell$ are close to the $\mu_i^\ell$, it remains to show that the $n - r_0$ eigenvalues are confined in a circle of radius $\thresh^\ell$. This is done in three parts:
\begin{enumerate}
  \item a \emph{tangle-free} decomposition, expressing $A^\ell$ as a product involving its expectation $Q$, the powers $A^t$ for $t \leqslant\ell$ and some additional random matrices $\underline A^{(t)}$, which can be understood as the centered versions of $A^t$.
  \item a trace method on the aforementioned centered matrices, inspired by \cite{furedi_eigenvalues_1981}, to bound their spectral radius;
  \item finally, a scalar product bound to control the whole sum whenever $x \in H^\bot$.
\end{enumerate}

\paragraph{Tangle-free decomposition.} We showed in Lemma \ref{tanglefree} that with high probability the graph $G$ is $2\ell$-tangle-free; as a result, for all vertices $u, v$ and $t \leqslant\ell$ we have

\[ A^t_{uv} = \sum_{\gamma \in F_{uv}^{(t)}} \prod_{s = 1}^t A_{\gamma_{s-1}\gamma_s}, \]
where the sum ranges over all tangle-free paths (i.e. paths whose induced graph is tangle-free) of length $t$ in the complete graph $K_n$. The centered matrices $\underline A^{(t)}$ are thus similarly defined as
\[ \underline A^{(t)}_{uv} = \sum_{\gamma \in F_{uv}^{(t)}} \prod_{s = 1}^t \underline A_{\gamma_{s-1}\gamma_s},\]
where $\underline A = A - Q$ is the centered version of $A$. 

To decompose $A^t$ in terms of the latter matrices, we make use of the following equality, valid for any $(a_i), (b_i)$:
\[ \prod_{s = 1}^t a_s = \prod_{s = 1}^t b_s + \sum_{k = 1}^t \left(\prod_{s = 1}^{k-1} b_s \right)(a_k - b_k)\left(\prod_{s = k+1}^t a_s\right)\]
Applying this to the two above equations yields
\[ A^\ell_{uv} = \underline A^{(\ell)}_{uv} + \sum_{k = 1}^\ell \sum_{\gamma \in F_{uv}^{(\ell)}}\left(\prod_{s = 1}^{k-1} \underline A_{\gamma_{s-1}\gamma_s} \right)(Q_{\gamma_{k-1}\gamma_k})\left(\prod_{s = k+1}^\ell A_{\gamma_{s-1}\gamma_s}\right)\]
Each term in the sum above is close to $[\underline A^{(k-1)}QA^{\ell - k - 1}]_{uv}$, with the following caveat : the concatenation of a path in $F_{uv}^{(k - 1)}$ and one in $F_{wx}^{(\ell - k - 1)}$ is not necessarily tangle-free ! Nevertheless, we write
\[ [\underline A^{(k-1)}QA^{\ell - k - 1}]_{uv} = \sum_{\gamma \in F_{uv}^{(\ell)}}\left(\prod_{s = 1}^{k-1} \underline A_{\gamma_{s-1}\gamma_s} \right)(Q_{\gamma_{k-1}\gamma_k})\left(\prod_{s = k+1}^\ell A_{\gamma_{s-1}\gamma_s}\right) + [R_k^{(\ell)}]_{uv}, \]
so that we finally get
\begin{equation}\label{eq:tangle_free_decomp}
  A^\ell = \underline A^{(\ell)} + \sum_{k = 1}^\ell \underline A^{(k-1)}QA^{\ell - k - 1} - \sum_{k = 1}^\ell R_k^{(\ell)}
\end{equation}

\paragraph{Bounding $\lVert\underline A^{(k)}\rVert$.} The trace method gets its name from its leverage of the following inequality:
\[ \left\lVert\underline A^{(k)}\right\rVert = \left\lVert\left(\underline A^{(k)}\underline A^{(k)*}\right)^m\right\rVert^{\frac 1 {2m}} \leqslant\mathrm{tr}\left[\left(\underline A^{(k)}\underline A^{(k)*}\right)^m\right]^{\frac1{2m}}. \]
The above trace can be expanded as
\begin{equation}\label{eq:trace_expansion}
  \mathrm{tr}\left[\left(\underline A^{(k)}\underline A^{(k)*}\right)^m\right] = \sum_{\gamma} \prod_{i = 1}^{2m} \prod_{t = 1}^k \underline A_{\gamma_{i, t-1}\gamma_{i, t}},
\end{equation}
where the sum ranges over all concatenations of $2m$ $k$-paths $\gamma = (\gamma_1, \dots, \gamma_{2m})$ such that $\gamma_i$ is tangle-free for all $i$, and with adequate boundary conditions.

The goal is now to use a Markov bound, and thus to compute the expectation in \eqref{eq:trace_expansion}; the key argument is the following:
\begin{center}
  \emph{Each term in the sum \eqref{eq:trace_expansion} has expectation zero unless $\gamma$ visits each of its edges at least twice.}
\end{center}
We now classify the subgraphs $\gamma$ by their number of vertices $v(\gamma)$ and edges $e(\gamma)$, and we say that $\gamma$ and $\gamma'$ are equivalent if there exists a permutation $\sigma \in \mathfrak S_n$ such that $\sigma(\gamma_{i, t}) = \gamma'_{i, t}$ for all $i, t$. All that remains is to bound the number of such equivalence classes and their contributions to the overall expectation; this is done in \cite{simon, ludovic} and yields the following results:
\begin{lemma}
  The number $\mathcal N(v, e)$ of equivalence classes of subgraphs $\gamma$ with $v$ vertices and $e$ edges such that each edge is visited at least twice satisfies
  \begin{equation}
    \mathcal N(v, e) \leqslant(2km)^{6m(e-v+1) + 2m},
  \end{equation}
  and for each $\tilde \gamma \in \mathcal N(v, e)$, the contribution $\mathcal W(\tilde \gamma)$ of the equivalence class to the trace expectation is bounded above:
  \[ \mathcal W(\tilde \gamma) \leqslant\mathcal W(v, e) \coloneqq  n^{v - e} \rho^e \left( \frac{d L^2}{\rho} \right)^{3(e-v) + 8m} \]
\end{lemma}
Now, all that remains is to sum the terms $\mathcal N(v, e) \mathcal W(v, e)$ over all possible choices of $v$ and $e$, to find the following bound:

\[ \left\lVert\underline A^{(k)}\right\rVert \preceq \thresh^k. \]
The operator norm of $R_k^{(\ell)}$ is bounded using similar arguments.

\paragraph{A scalar product bound.} Let $w \in \mathbb R^n$; with the previously established bounds, we have
\[ \left\lVert A^\ell w \right\rVert \preceq \thresh^\ell + \sum_{k = 1}^\ell \thresh^k \left\lVert QA^{\ell - k - 1}w \right\rVert\]
It remains to bound the rightmost norm whenever $w$ is orthogonal to the $A^\ell \psi_i$. First, we use the eigendecomposition of $Q$:
\begin{align*}
  \left\lVert QA^{\ell - k - 1}w \right\rVert &= \left\lVert \sum_{i \in [r]} \mu_i \varphi_i \psi_i^* A^{\ell - k - 1} w \right \rVert \\
  &\leqslant\mu_1\sum_{i \in [r]} \langle (A^*)^{\ell - k - 1}\psi_i, w \rangle
\end{align*}
Since $\langle w, (A^*)^{\ell}\psi_i \rangle = 0$ by assumption, we can use Cauchy-Schwarz and a telescopic sum to bound the scalar product:
\[ \langle (A^*)^{\ell - k - 1}\psi_i, w \rangle \leqslant|\mu_i|^{\ell - k - 1}\sum_{t = \ell - k - 1}^{\ell - 1} |\mu_i|^{-t} \left \lVert (A^*)^{t}\psi_i - \mu_i^{-1}(A^*)^{t+1}\psi_i \right \rVert \]
The final bound thus stems from the following lemma \cite{simon}:
\begin{lemma}
  For every $t \leqslant\ell$ and $i \in [r_0]$ we have
  \[ \left \lVert A^{t}\varphi_i - \mu_i^{-1}A^{t+1}\varphi_i \right \rVert^2 \preceq \thresh^{2t}, \]
  and the same holds for $A^*$ and $\psi_i$.
\end{lemma}
Indeed, adopting again the notations from \ref{sec:poisson}, we have to compute the expectation on the tree of the $t+1$-local function
\[ f(T, x) = (\tilde U_i(x, t) - \mu_i^{-1}\tilde U_i(x, t+1))^2, \]
which can be understood as the increment variance of the martingale $\tilde U_i(x, t)$. Subsequently, we can use the same arguments as in \ref{sec:correlation} with $i = j$, which shows
\[ \mathbf E[f(T, x)] = [K^{t+1}\varphi^{i, i}](x) \preceq \frac{\thresh^{2t}}{n}. \]
Summing this for $x \in [n]$ and applying Proposition \ref{lem:coupling_functionals}, we are done.

\section{Master Theorem for the stochastic block model}\label{app:SBM}

We recall the definition of $P$: given the cluster membership functions $\sigmal, \sigmar: [n] \to [r]$ and a connectivity matrix $F$ of size $r \times r$, the entries of $P$ are given by
\[ P_{x,y} = \frac{F_{\sigmal(x),\sigmar(y)}}{n}. \]
We introduce the probability vectors $\pl, \pr$, which are equal to the relative clusters sizes, as well as the `cluster intersection' matrix $\Pi$:
\[ (\pl)_i = \frac{\card(\sigmal^{-1}(i))}n, \quad (\pr)_i = \frac{\card(\sigmar^{-1}(i))}n \quad \text{and}\quad \Pi_{i,j} = \frac{\card(\sigmal^{-1}(j) \cap \sigmar^{-1}(i))}{n}.\]
The cluster membership matrices $\Sigmal, \Sigmar$ are matrices of size $n \times r$, defined by
\[ (\Sigmal)_{x, i} = \mathbf{1}_{\sigmal(x) = i} \quad \text{and}\quad (\Sigmar)_{x, i} = \mathbf{1}_{\sigmar(x) = i}. \]
With these notations, the following identities hold:
\begin{equation}\label{eq:sigma_identities}
\begin{aligned}
  &\pl = \frac{1}{n}\Sigmal \mathbf 1 , && \diag(\pl) = \frac{1}{n}(\Sigmal)^* \Sigmal, && P  = \frac{1}{n}\Sigmal F (\Sigmar)^*; \\
  & \pr = \frac{1}{n}\Sigmar \mathbf 1  &&\diag(\pr) = \frac{1}{n}(\Sigmar)^* \Sigmar,&& \Pi = \frac{1}{n} (\Sigmar)^* \Sigmal\\
\end{aligned}
\end{equation}
\subsection{Spectral decomposition of \texorpdfstring{$P$}{P}}

We prove in this section a slightly refined version of Proposition \ref{prop:spectral_sbm}.
\begin{proposition}\label{prop:better_spectral_sbm}
  The non-zero eigenvalues of $P$ are exactly those of the modularity matrix $F \Pi$, with the same multiplicities. Further, each \emph{right} eigenvector of $P$ is of the form $\Sigmal f$, where $f$ is a right eigenvector of $F \Pi$, while each \emph{left} eigenvector of $P$ is of the form $\Sigmar g$ with $g$ a left eigenvector of $\Pi F$.
\end{proposition}
The proof of this proposition relies on the following elementary lemma, a consequence of the Sylvester identity $\det(z-XY)=\det(z-YX)$.
\begin{lemma}
  Let $X$ be a $n \times m$ matrix and $Y$ a $m \times n$ matrix. Then the non-zero eigenvalues of $XY$ are the same as those of $YX$, with identical multiplicities.
\end{lemma}

\begin{proof}[of Proposition \ref{prop:better_spectral_sbm}]
We apply the above lemma to $X = \Sigmal$ and $Y = \frac1n F (\Sigmar)^*$; the identities in \eqref{eq:sigma_identities} show that $XY = P$ and $YX = F\Pi$, which directly gives the desired result. Now, let $f$ be a right eigenvector of $F\Pi$, with associated eigenvalue $\lambda$, and define $\varphi = \Sigmal f$. Then
\[ P\varphi = \frac1n \Sigmal F(\Sigmar)^*\Sigmal f = \Sigmal F \Pi f = \lambda \Sigmal f = \lambda \varphi. \]
Combined with the previous result, this implies that all right eigenvectors of $P$ with non-zero eigenvalues are of the form $\Sigmal f$ for an eigenvector $f$ of $F\Pi$. In particular, they are constant on the left clusters. The result on left eigenvectors is proved similarly.
\end{proof}

Let $f_1, \dots, f_r$ (resp. $g_1, \dots, g_r$) be a basis of right (resp. left) eigenvectors of $F\Pi$ (resp. $\Pi F$), not necessarily normalized. We define the entrywise products $f^{i, j}$ and $g^{i, j}$ as in equation \eqref{eq:phi_ii_def}. The following statement describes the unit eigenvectors of $P$ in terms of $f_i, g_j$. 
\begin{lemma}\label{lem:phi_xi_sbm}
  Let $(\varphi_i)$ (resp. $(\xi_i)$) be a basis of normed right (resp. left) eigenvectors of $P$. Then
  \[ \varphi_i = \frac{\Sigmal f_i}{\sqrt{n\langle \pl, f^{i,i} \rangle}} \qquad \text{and} \qquad \xi_i = \frac{\Sigmar g_i}{n\sqrt{n\langle \pr, g^{i,i} \rangle}}.\]
\end{lemma}

\begin{proof}
  In light of Proposition \ref{prop:better_spectral_sbm}, we only have to compute the norms of $\Sigmal f_i$ and $\Sigmar g_i$:
  \begin{align*} 
    \left| \Sigmal f_i \right|^2 &= f_i^* (\Sigmal)^* \Sigmal f_i = n f_i^* \diag(\pl) f_i = n \langle \pl, f^{i, i} \rangle.
  \end{align*}
  The first equality follows, and the second is proved in identical fashion.
\end{proof}

\subsection{Master Theorem for SBM}
We are now ready to prove the version of the Master Theorem, adapted to the directed SBM.

\begin{theorem}\label{thm:master_sbm}
  Let $r_0$ be the number of eigenvalues $\nu_i$ such that $\nu_i^2 > \nu_1$. Then, with high probability the following holds: the $r_0$ highest eigenvalues $\lambda_1, \dots, \lambda_{r_0}$ of $A$ satisfy
  \[ |\lambda_i - \nu_i| = o(1), \]
  and all other eigenvalues of $A$ are asymptotically smaller that $\sqrt{\nu_1}$. Further, if $v_i, u_i$ are a pair of left/right unit eigenvectors of $A$ associated with $\lambda_i$, then
  \[ |\langle u_i, \varphi_j \rangle| = a_{i, j} + o(1) \qquad \text{and}\qquad |\langle v_i, \xi_j \rangle| = b_{i,j} + o(1), \]
  where $a_{i,j}$ and $b_{i, j}$ are defined as
  \begin{align} 
    a_{i, j} &= \frac{\left|\left\langle \pl, f^{i, j} \right\rangle\right|}{\sqrt{\left\langle \pl, f^{j, j} \right\rangle\left \langle \pl, (I - \nu_i^{-2} F\Pi)^{-1}f^{i, i} \right \rangle}} \label{eq:sbm_aij}\\
     b_{i, j} &= \frac{\left|\left\langle \pr, g^{i, j} \right\rangle\right|}{\sqrt{\left\langle \pr, g^{j, j} \right\rangle\left \langle \pr, (I - \nu_i^{-2} (\Pi F)^*)^{-1}g^{i, i} \right \rangle}}\label{eq:sbm_bij}
  \end{align}
\end{theorem}

The first part of this theorem is an application of Theorem \ref{thm:main}, by means of Proposition \ref{prop:better_spectral_sbm}. It remains to compute the $a_{i, j}$ and $b_{i, j}$ as a function of the SBM parameters. We recall that the definition of $\Gamma$ and $\Gamma^*$ are in \eqref{gamma}. 

\begin{lemma}\label{lem:gamma_sbm}
  Let $z \in \mathbb C$, and $h \in \mathbb R^r$. Then
  \[ \Gamma(z, \Sigmal h) = n \left \langle \pl, (I - z^{-1} F\Pi)^{-1}h \right \rangle \qquad \text{and} \qquad \Gamma^*(z, \Sigmar h) = n \left \langle \pr, (I - z^{-1} (\Pi F)^*)^{-1}h \right \rangle. \]
\end{lemma}

\begin{proof}
  Since the graph is unweighted, we have $K = P$. By an immediate recursion, we have
  \[ P^t \Sigmal h = \Sigmal (F \Pi)^t h, \]
  so that using the first identity of \eqref{eq:sigma_identities}
  \[ \langle \mathbf 1, P^t \Sigmal h \rangle = n \langle \pl, (F \Pi)^t h \rangle. \]
  Summing over all $t$ and using the Von Neumann summation implies the first equality, and the second is alike.
\end{proof}

As a result,
\[ \nu_i^2 R_i = \Gamma(\nu_i^2, \varphi^{i, i}) = \Gamma(\nu_i^2, \Sigmal f^{i, i}) = n \left \langle \pl, (I - \nu_i^{-2} F\Pi)^{-1}f^{i, i} \right \rangle, \]
where we used the previous lemma, and similarly
\[ \nu_i^2 L_i = n \left \langle \pr, (I - \nu_i^{-2} (\Pi F)^*)^{-1}g^{i, i} \right \rangle. \]

\begin{proof}[of Theorem \ref{thm:master_sbm}]
  Using the expressions in Theorem \ref{thm:main}, we have
  \[ a_{i,j} = \frac{|\langle \varphi_i, \varphi_j \rangle |}{|\nu_i|\sqrt{R_i}} \qquad \text{and} \qquad b_{i,j} = \frac{|\langle \xi_i, \xi_j \rangle |}{|\nu_i|\sqrt{L_i}}. \]
  Computing the numerators is straightworward using Lemma \ref{lem:phi_xi_sbm}:
  \begin{align*}
    \langle \varphi_i, \varphi_j \rangle &= \frac{\left\langle \Sigmal f_i, \Sigmal f_j \right\rangle}{n \sqrt{\langle p, f^{i, i} \rangle \langle \pl, f^{j, j} \rangle}} \\
    &= \frac{\langle \pl, f^{i, j} \rangle}{\sqrt{\langle \pl, f^{i, i} \rangle \langle \pl, f^{j, j} \rangle}}.
  \end{align*}
  On the other hand, for the denominator, we have
  \[ \nu_i^2 R_i = \Gamma(\nu_i^2, \varphi^{i, i}) = \frac{\Gamma(\nu_i^2, \Sigmal f^{i, i})}{n \langle \pl, f^{i, i} \rangle}, \]
  and using Lemma \ref{lem:gamma_sbm} we find
  \[ \nu_i^2 R_i = \frac{\left \langle \pl, (I - \nu_i^{-2} F\Pi)^{-1}f^{i, i} \right \rangle}{\langle \pl, f^{i, i} \rangle} \]

  It simply remains to simplify the expressions to prove the formula for $a_{i, j}$, and the exact same method works for $b_{i, j}$ as well.
\end{proof}

\section{Pathwise SBM}\label{app:pathwise}

In this section, we derive the thresholds shown in Subsection \ref{subsec:cycle}. Let us first give some motivation on this model.

\subsection{Motivation}
Stochastic block-models with  a pathwise structure as in \eqref{def:Fcyclic} are well suited for modeling flow data: the intra-block connectivity is the same $s/2$ in any blocks; connections can only happen between adjacent blocks and the rate depends on the flow order: edges have a higher chance of appearing from one block $V_i$ to the following $V_{i+1}$, than between one block $V_i$ and the preceding one $V_{i-1}$ ($\eta$ versus $1 - \eta$). The model in \cite{laenen2020higherorder} is a small variant of this one: in their model, undirected edges appear between adjacent blocks, and then one direction is chosen uniformly at random with probability $\eta$ for edges between adjacent blocks, and with probability $1/2$ for edges inside the same block. Our model allows the appearance of a double edge $(x,y), (y,x)$, which is not the case in their model. However, in the sparse regime where $s$ does not depend on $n$, the two models are contiguous and our results can be shown to hold for both. 

\begin{remark}
The works \cite{van2015spectral, van2019spectral} are close in spirit to ours, although they do not approach sparse regimes. We chose to perform the computations for the specific $F$ above, but the same computations can be done for other models. In particular, it would be interesting to perform these computations for the matrix given in \cite{van2015spectral}, p. 73 and to recover the shape observed by the author in Figs 2.23-24.
\end{remark}

\subsection{Model density}

Let us compute the mean degree $d$ in this model, when there are $k$ blocks and the asymmetry parameter is $\eta$. All the blocks have the same size $n/k$, hence they have $(n/k)^2$ entries; on the $k$ diagonal blocks, the mean degree is $s/2$, on the $k-1$ upper-diagonal blocks it is $s(1-\eta)$ and on the $k-1$ lower diagonals they are $s\eta$, so 
\begin{align}\label{pathwise:meandegree}
d =\frac{1}{n^2}\left[ k\frac{s}{2}\frac{n^2}{k^2} + (k-1) s(1-\eta) \frac{n^2}{k^2} + (k-1)s\eta \frac{n^2}{k^2} \right] = \frac{s}{k} \left( \frac{3}{2} - \frac{1}{k^2}\right).
\end{align}
For each $k$, the unique parameter $s$ such that the model has mean degree $d$ is given by $s(k,d) = kd(3/2 - k^{-2})^{-1}$. Table \ref{tab:special_numbers} gives the values of $s$ used in our simulations in Figure \ref{fig:results}.

\begin{table}[H]\centering
\begin{tabular}{r|ccc}
\hline \hline
\textbf{requested mean degree $d$} & \textbf{2} & \textbf{3} & \textbf{4} \\
\hline 
number of blocks $k=2$ & 3.2 & 4.8 & 6.4 \\
$k=4$ & 5.5 & 8.3 & 11.1 \\
$k=6$ & 8.1 & 12.2 & 16.3 \\
\hline \hline
\end{tabular}
\caption{Value of $s=s(k,d)$ such that, for the given number of blocks $k$, the mean degree of the model is equal to $d$.}\label{tab:special_numbers}
\end{table}

\subsection{Eigendecomposition of tridiagonal Toeplitz matrices. }\label{app:tridiag}

Since the blocks on the right and on the left are identical and have the same size $n/r$, the formulas in Theorem \ref{thm:master_sbm} are really easy to use. 

Let $F$ be the tridiagonal $k \times k$ Toeplitz matrix defined in \eqref{def:Fcyclic}. We extract the following formulas from \cite{pasquini2006tridiagonal} (see (4) for eigenvalues and (7)-(8) for eigenvectors). The $k$ eigenvalues are
\begin{equation}\label{app:tridiag_eigs}\tilde \nu_k =  \frac{s}{2} +2s \cos\left( \frac{\pi k}{r+1}\right)\sqrt{\eta(1-\eta)} \qquad \qquad  (1 \leqslant k \leqslant r).\end{equation}
and the corresponding right-eigenvectors $f_i$ and left-eigenvectors $g_i$ are
\begin{equation}\label{app:eigenvectors_F}
  f_i(j) \propto \left(\frac{1-\eta}{\eta}\right)^{j/2}\sin\left( \frac{ij\pi}{r+1}\right) \qquad \qquad g_i(j) \propto \left(\frac{\eta}{1-\eta}\right)^{j/2}\sin\left( \frac{ij\pi}{r+1}\right)
\end{equation}

\subsection{Digression: the full threshold}

We hereby state an auxiliary result that might be of potential interest, which simply consists in an application of Theorem \ref{thm:master_sbm}. 

\begin{proposition}\label{thm:cyclic_threshold}
  In the pathwise SBM as in \eqref{def:Fcyclic} with $r$ blocks, degree parameter $s>1$ and asymmetry parameter $1/2 \leqslant \eta \leqslant 1$, the $r$ distinct eigenvalues/eigenvectors can be detected if 
  \begin{equation}\label{threshold}
    \frac{s}{r}> \frac{1/2+c_1\theta}{\min_{k \in [r]}(1/2+c_k\theta)^2}
  \end{equation}
  where $\theta = 2\sqrt{\eta(1-\eta)}$ and $c_k = \cos(k\pi/(r+1))$.
\end{proposition}

We plotted the threshold for several values of $r$ in Figure \ref{fig:thresholds}. It is interesting to note that for specific values of $\theta$, the threshold for $s$ is $+ \infty$; this corresponds to cases where $\theta = -2c_k$, and one eigenvalue is zero. This is a good illustration of the principle discussed earlier : $r_0 = 1$ suffices to recover cluster information (since the top eigenvector is nonconstant), even though it's completely impossible to recover as many informative eigenvectors as there are clusters.

\begin{figure}\centering
   \includegraphics[width= 0.9\textwidth]{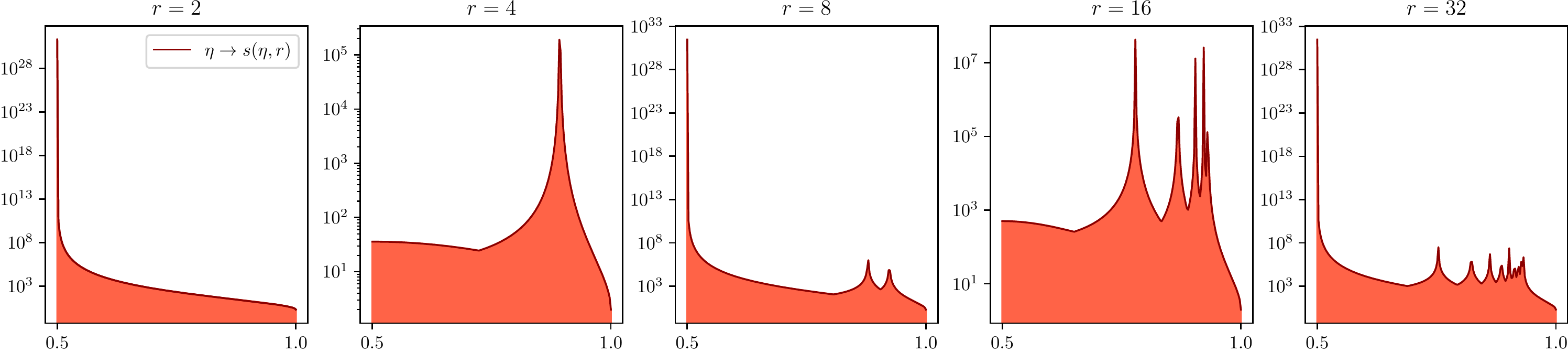}
   
   \vspace*{1cm}

     \includegraphics[width = 0.4\textwidth]{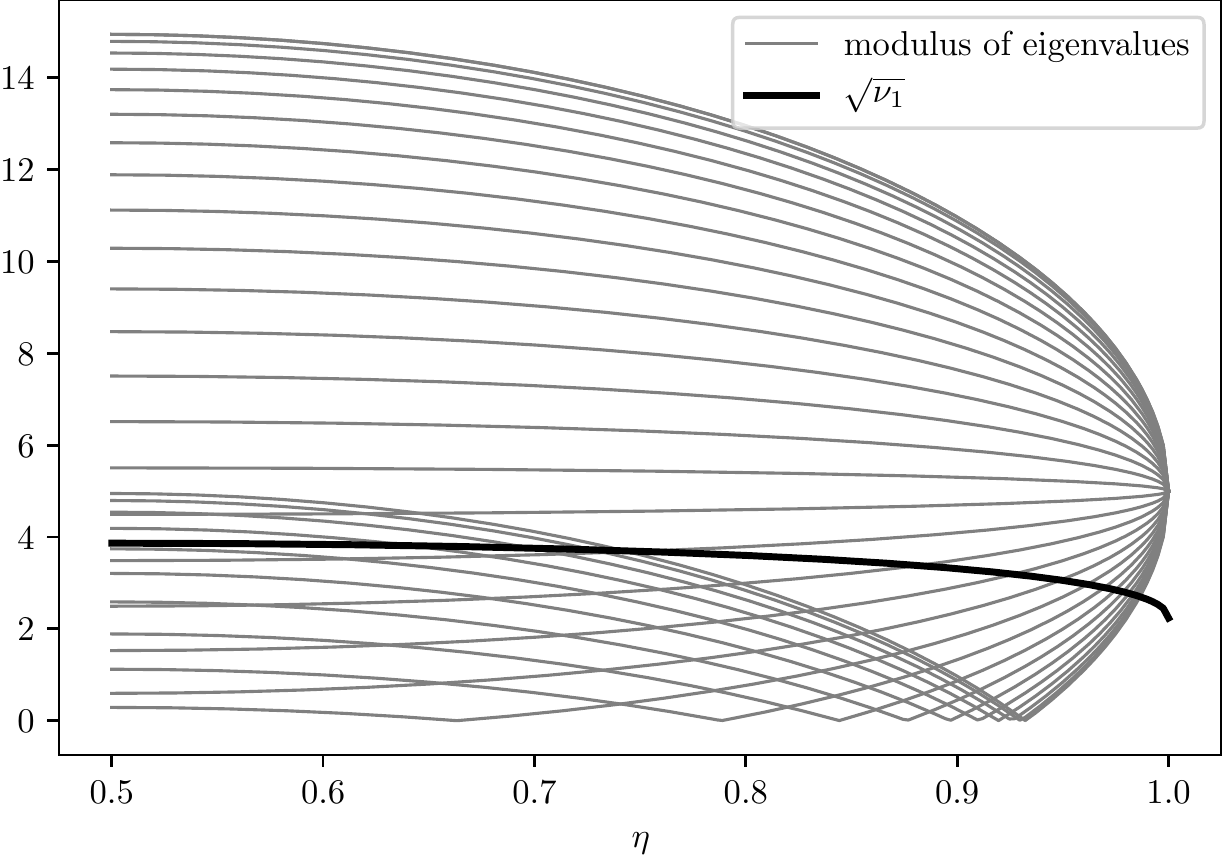}  \includegraphics[width=0.4\textwidth]{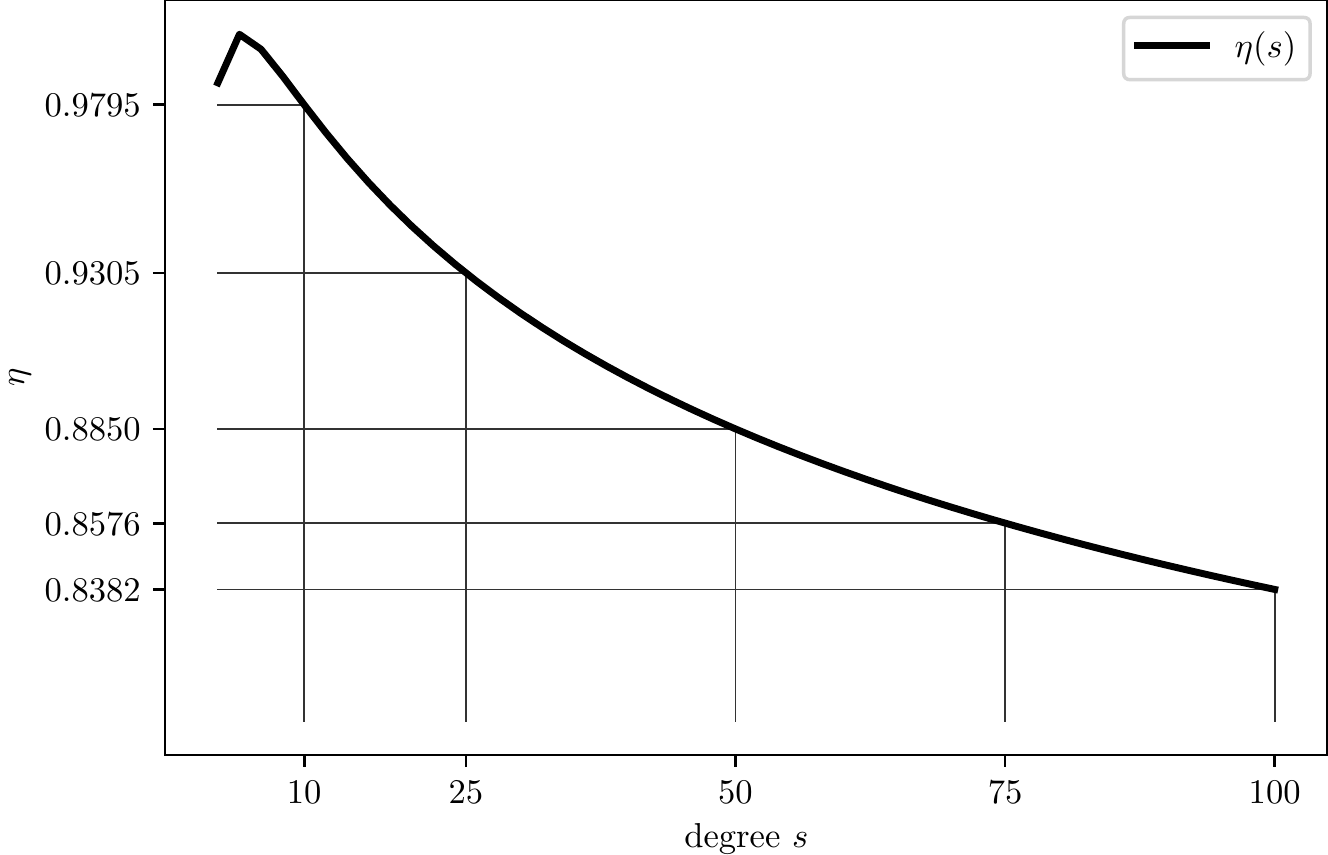}

  \caption{\textbf{Top}: A plot of the shape of the RHS of \eqref{threshold} as a function of $\eta$, for different values of $r$ ranging from $2$ to $2^5$, with a log-scale on the $y$-axis. The red zone represent the set of degree parameters $s$ that lie below the threshold $s(\eta, r)$ in \eqref{threshold}: they are the values for which our method does not yield full reconstruction guarantees, since at least one eigenvalue/eigenvector couple is `lost in the bulk' of the spectrum of $A$. \newline \textbf{Bottom left}: Here $s=10$ and $r=30$. The absolute values of the eigenvalues $\nu_k = s(1/2+c_k\theta)$ are plotted in grey, while the threshold $\rho = \sqrt{\nu_1}$ is in bold black. \newline \textbf{Bottom right}: the shape of $\eta \mapsto \eta(s)$. }\label{fig:thresholds}
\end{figure}

\subsection{Computations for the two-block case}

We place ourselves in the setting of Theorem \ref{thm:master_sbm}; recall that
\begin{equation}
  F = \begin{bmatrix}
  s/2 & s\eta \\ s(1-\eta) & s/2
  \end{bmatrix}, \qquad \qquad F\Pi = \Pi F = \begin{bmatrix}
  s/4 & s\eta/2 \\ s(1-\eta)/2 & s/4
  \end{bmatrix} \eqqcolon  M.
\end{equation}

Define $\theta = 2\sqrt{\eta(1-\eta)}$; as in \eqref{app:tridiag_eigs}, the eigenvalues of $M$ are
\[ \nu_1 = \frac{s(1+\theta)}{4} \qquad \text{and} \qquad \nu_2 = \frac{s(1 - \theta)}4,\]
with associated right and left eigenvectors
\begin{align*}
  f_1 &= \left(\sqrt{\eta}, \sqrt{1-\eta}\right) & f_2 &= \left(\sqrt{\eta}, -\sqrt{1-\eta}\right) \\
  g_1 &= \left(\sqrt{1-\eta}, \sqrt{\eta}\right) & g_2 &= \left(\sqrt{1 - \eta}, -\sqrt{\eta}\right).
\end{align*}

Applying Theorem \ref{thm:master_sbm}, we have $r_0 = 2$ whenever $\nu_2 ^2 > \nu_1$, which simplifies to
\[ s > \frac{4(1 + \theta)}{(1 - \theta)^2}, \]
which settles the first part of Theorem \ref{thm:long}. Now, simplifying \eqref{eq:sbm_aij} whenever $i = j$, we have
\[ a_{i, i} = \sqrt{\frac{\left\langle p, f^{i, i} \right\rangle}{\left \langle p, (I - \nu_i^{-2} M)^{-1}f^{i, i} \right \rangle}},\]
and since $p = (1/2, 1/2)$ and the $f_i$ have unit length, this simplifies further to
\begin{equation}
  a_{i, i} = \frac1{\sqrt{\left \langle \mathbf 1, (I - \nu_i^{-2} M)^{-1}f^{i, i} \right \rangle}}.
\end{equation}
The standard adjoint formula yields, whenever $\alpha < 1/\nu_1$,
\[(I - \alpha M)^{-1} = \frac{1}{4(1 - \alpha \nu_1)(1 - \alpha \nu_2)}\begin{pmatrix}
  4 - \alpha s & 2\alpha s \eta \\
  2 \alpha s (1 - \eta) & 4 - \alpha s,
\end{pmatrix} \]
and since $f^{1, 1} = f^{2, 2} = (\eta, 1 - \eta)$, we have
\begin{align*}
  \left \langle \mathbf 1, (I - \alpha M)^{-1}f^{i, i} \right \rangle ^{-1} &= \frac{4(1 - \alpha \nu_1)(1 - \alpha \nu_2)}{4 - \alpha s + \alpha s \theta^2}\\
  &= \frac{4 - 2\alpha s + \alpha^2 s^2 \frac{1 - \theta^2}4}{4 - \alpha s + \alpha s\theta^2}\\
  &\eqqcolon  \gamma(\alpha s).
\end{align*}
Since we will choose $\alpha = \nu_i^{-2}$, we have $\alpha s -> 0$ and
\begin{align*}
  \gamma(x) &= \left(1 - \frac x 2\right)\left(1 + \frac{(1 - \theta^2)x}4 \right) + O(x^2) \\
  &= 1 - \frac{1 + \theta^2}4 x + O(x^2).
\end{align*}
Substituting $\alpha = \nu_i^{-2}$ and taking the square root, we find
\[ a_{i, i} = 1 - \frac{1 + \theta^2}8 \frac{s}{\nu_i^2} + O\left( \frac 1 {s^2} \right) = 1 - \frac 2 s \frac{1 + \theta^2}{(1 \pm \theta)^2} + O\left( \frac 1 {s^2} \right),\]
which are the expressions in Theorem \ref{thm:long}.

Note that it is possible to continue the computations and find explicit expressions for the $a_{i, i}$ in terms of $\eta$ and $s$, but the resulting expressions are too complex to give any more insight than the asymptotic expressions.

\subsection{Computations when there are two blocks}
We defined $\eta(s)$ as the unique number in $[0,1]$ such that
\[ s > \frac{4(1 + \theta)}{(1 - \theta)^2}\quad \Longleftrightarrow \quad \eta > \eta(s), \]
with $\theta(\eta) = 2\sqrt{\eta(1-\eta)}$ (see also Figure \ref{fig:thresholds}). The inverse $\theta^{-1}:[0,1] \to [1/2, 1]$ is given by
\[\theta^{-1}(t) = \frac{1 + \sqrt{1-t^2}}{2}.\]
Now, the solution $x$ of the equation
\[s = \frac{4(1+x)}{(1-x)^2}\]
is the solution of the quadratic $sx^2 - (2s+4)x+s-4=0$. The discriminant is $\Delta(s) = 16(2s+1)$ and the unique solution in $[0,1]$ is
\[x(s) = 1 + \frac{2 - 2\sqrt{2s+1}}{s}.\]
Finally, the smallest $\eta(s)$ for which \eqref{threshold} is satisfied is $\eta(s) = \theta^{-1}(x(s))$, that is,
\begin{equation}\label{eq:etas_expression}
  \eta(s) = \frac{1+\sqrt{1 - x(s)^2}}{2} = \frac{1+\sqrt{1 - \left( 1 + \frac{2 - 2\sqrt{2s+1}}{s} \right)^2}}{2}.
\end{equation}
It is possible to expand the term inside the square root, but with no meaningful gain. The function $\eta(s)$ has the series expansion
\[ \eta(s) = \frac12 + \sqrt[4]{\frac2s} + O(s^{-3/4}), \]
but the convergence is very slow: the truncated RHS is less than one only whenever $s \geq 32$.

\section{Convergence of eigenvectors}\label{sec:last}

The goal of this section is to prove Theorem \ref{thm:eigenvector_convergence}. We place ourselves in the stochastic block model setting as in Section \ref{subsec:SBM}, with $\sigmal = \sigmar = \sigma$ and $\pr = p$. We assume that $F$ and $p$ are constant with $n$, so that the modularity matrix $M$ doesn't depend on $n$. 

\subsection{Convergence of \texorpdfstring{$\tilde U_i(x)$}{tree pseudo-eigenvectors}}

We consider the multitype Galton-Watson trees $(\mathscr T_j, o)$ as defined in \cite{BLM}: the root $o$ has \emph{type} $\sigma(o) = j$, and afterwards, each vertex with type $i$ has $\mathrm{Poi}(M_{i, k})$ children of type $k$. Unlike our initial trees $T_x$, which are heavily $n$-dependent with node labels in $[n]$ and edge probabilities $P_{x,y}$, the tree $\mathscr{T}_j$ does not depend on $n$, and its labels are in $[r]$.  We define on those trees the random processes
\[ \mathscr U_i(j, t) = \nu_i^{-t}\sum_{k = 1}^r  N_k(\mathscr T_j, t) f_i(k), \]
where $ N_k(\mathscr T_j, t)$ counts the number of vertices of type $k$ at depth $t$ in $\mathscr T_j$. Assume that we have chosen the $f_i$ such that $\phi_i = \Sigmar f_i$ has unit norm; we recall that the processes $\tilde{U}_i(x,t)$ were defined in \eqref{eq:tilde_ui}-\eqref{martingale}. 
\begin{lemma}
  Let $x\in [n]$, and define $j = \sigma(x)$. Then the processes $\tilde U_i(x, t)$ and $\mathscr U_i(j, t)$ have the same distribution.
\end{lemma}

\begin{proof}
  Let $\sigma(T_x)$ be the tree where a vertex with label $y$ is mapped to a vertex with label $\sigma(y)$. Then the root of $\sigma(T_x)$ has label $\sigma(x) = j$. Take a vertex in $T_x$ with label $y$, and let $\sigma(y) = i$; the number of children of $y$ with type $k$ has distribution $\mathrm{Poi}(\tilde M_{y k})$, with
  \[ \tilde M_{y, k} = d_y \sum_{\sigma(z) = k} \frac{P_{yz}}{d_y}.\]
  Using the definition of $P$ for the SBM, we have
  \[ \tilde M_{y, k} = \sum_{\sigma(z) = k} \frac{F_{\sigma(y), \sigma(z)}}n = F_{i, k} p_k = M_{i, k}. \]
  Therefore, the laws of $\sigma(T_x)$ and $\mathscr T_j$ coincide. Now, since there are no weights the product in \eqref{eq:tilde_ui} is equal to $1$ and we have
  \[ \tilde U_i(x, t) = \sum_{x_t} \varphi_i(\iota(x_t));\]
  replacing $\varphi$ by its definition in terms of $f_i$,
  \begin{align*} 
    \tilde U_i(x, t) &= \nu_i^{-t}\sum_{x_t} f_i(\sigma(\iota(x_t))) = \nu_i^{-t}\sum_{k = 1}^{r}  N_k(\sigma(T_x), t) f_i(k),
  \end{align*}
  which ends the proof.
\end{proof}

This lemma allows us to translate results back and forth between the $T_x$ and the $\mathscr T_{\sigma(x)}$; we thus know from the expectation/correlation computations in Subsections \ref{sec:martingale_equation}-\ref{sec:correlation} that $\mathscr U_i(j, t)$ (with $j = \sigma(x)$) is a martingale with
\[ \mathbf E[\mathscr U_i(j, t)] = f_i(j) \quad \text{and} \quad \mathbf E[\mathscr U_i(j, t)^2] \leqslant \Gammar(i, i).\]

By the Doob martingale convergence theorem, this implies that $\mathscr U_i(j, t)$ converges in $\mathrm L^2$ as $t \to \infty$ towards a random variable $\mathscr Z_{i,j}$. 
Since we have convergence in $\mathrm L^2$, it entails
\begin{align*}
  \mathbf E[\mathscr Z_{i,j}] &= \lim_{t \to +\infty} \mathbf E[\mathscr U_i(j, t)] = f_i(j)\\
  \mathbf E[\mathscr Z_{i,j}^2] &= \lim_{t \to +\infty} \mathbf E[\mathscr U_i(j, t)^2] = \left[(I - \nu_i^{-2}M)^{-1}f^{i, i} \right](j).
\end{align*}
In the last line, we used computations from Appendix \ref{app:SBM} to determine the limit.

Another important fact is that the law of $\mathscr U_i(j, t)$ does not depend on $n$ whatsoever. This implies that $\mathscr U_i(j, \ell)$ converges to $\mathscr Z_{i,j}$ as $n \to +\infty$, which in turn yields the following proposition.

\begin{proposition}\label{prop:convergence_tilde_ui}
  Let $ Z_{i, j}$ be the limit of the random process $\mathscr U_{i}(j, t)$. Then, 
\begin{equation}
\tilde U_{i}(x) \xrightarrow[n \to \infty]{\mathrm{L}^2}\mathscr Z_{i,\sigma(x)}.
\end{equation}  
\end{proposition}

\subsection{Convergence of the pseudo-eigenvectors}

Now that we showed convergence on the random tree, we shall use the concentration proposition \ref{lem:coupling_functionals} to translate it on the graph. Let $h: \mathbb R \to \mathbb R$ be a bounded continuous function, and define
\[ f(G, x) = \frac 1{p_j n}\, \mathbf 1_{\sigma(x) = j}\, h(U_i(x)). \]
The function $f$ satisfies the hypotheses of Proposition \ref{lem:coupling_functionals} since $h$ is bounded, and we have
\[ f(T_x, x) = \frac 1{p_j n}\, \mathbf 1_{\sigma(x) = j}\,h(\tilde U_i(x)) \]
Using Proposition \ref{prop:convergence_tilde_ui}, and the fact that $\mathrm L^2$ convergence implies convergence in distribution,
\[ \mathbf E[\mathbf 1_{\sigma(x) = j}\,h(\tilde U_i(x))] \to \mathbf 1_{\sigma(x) = j} \mathbf E[h(\mathscr Z_{i, j})] \]
uniformly in $x$, so that summing over all vertices
\begin{equation}\label{eq:convergence_ui}
  \lim_{n \to + \infty} \frac{1}{p_j n}\sum_{\sigma(x) = j} h(U_i(x)) = \mathbf E[h(\mathscr Z_{i, j})].
\end{equation}
The above equation implies immediately that the discrete distribution of the $U_i(x)$ with $\sigma(x) = j$ converges weakly to $ \mathscr Z_{i, j}$; in other words, 
\begin{equation}
\label{app:discrete_convg}
\frac{1}{p_jn}\sum_{\sigma(x) = j} \delta_{U_i(x)} \xrightarrow[n \to \infty]{\mathrm{d}} \mathscr Z_{i,j}.
\end{equation}

\subsection{Convergence of the eigenvector}
 Define the normalized pseudo eigenvectors:
\[ \bar U_i = \frac{\sqrt{n} U_i}{|U_i|}, \]
then thanks to \eqref{UU} we have $|U_i| - \sqrt{n}\gamma \ll 1$, where 
\[ \gamma = \sqrt{ \left \langle p, (I - \nu_i^{-2}M)^{-1} f^{i, i} \right \rangle } \]
from the computations in Appendix \ref{app:SBM}. A consequence of our notation $\ll$ is that $\sqrt{n}/|U_i|$ converges in probability towards $\gamma>0$. Thanks to \eqref{app:discrete_convg} and Slutsky's lemma, 
\[ \frac 1{p_j n}\sum_{\sigma(x) = j} \delta_{\bar U_i(x)} \xrightarrow[~ n \to \infty~]{\mathrm{d}} Z_{i,j}\coloneqq  \frac{\mathscr{Z}_{i,j}}{\gamma}, \]
and consequently $Z_{i,j}$ has mean and variance
\begin{equation}\label{eq:muij_sigmaij}
  \mu_{i,j} = \frac{f_i(j)}{\sqrt{\left \langle p, (I - \nu_i^{-2}M)^{-1} f^{i, i} \right \rangle}} \qquad \text{and} \qquad \sigma_{i,j}^2 = \frac{\left[(I - \nu_i^{-2}M)^{-1}f^{i, i} \right](j)}{\left \langle p, (I - \nu_i^{-2}M)^{-1} f^{i, i} \right \rangle}.
\end{equation}

Now comes the last step of our proof; let $\bar u_i $ be a right eigenvector of $A$ such that $|\bar u_i| = \sqrt{n}$ (one can take $\bar u_i= \sqrt{n} u_i$).
Then, \eqref{app:eigvec_cv} implies that
\begin{equation}\label{last} \left|\bar u_i - \bar{U}_i \right| \ll \sqrt{n}.\end{equation}
We want to show convergence for $\bar u_i$; using the Portmanteau lemma, it suffices to show bounds like \eqref{eq:convergence_ui} for Lipschitz functions. Let $h$ be a bounded function with Lipschitz constant $C$; we write
\begin{align*}
  \left|\frac1{p_j n} \sum_{\sigma(x) = i} [h(\bar u_i(x)) - h(\bar U_i(x))] \right| &\leqslant \frac 1 {p_jn}\sum_{\sigma(x) = i} \left|h(\bar u_i(x)) - h(\bar U_i(x))\right| \\
  &\leqslant \frac{C}{p_j n}\sum_{\sigma(x) = i} \left|\bar u_i(x) - \bar U_i(x)\right| \\
  &\leqslant \frac{C}{p_j n} \sqrt{p_j n} \left|\bar u_i - \bar U_i \right|,
\end{align*}
using the Cauchy-Schwarz inequality. Using \eqref{last}, we are done.

\end{document}